\PassOptionsToPackage{dvipsnames}{xcolor}
\documentclass[review,onefignum,onetabnum]{siamart171218}

\usepackage{lipsum}
\usepackage{amsmath}
\usepackage{amssymb}
\usepackage{amsfonts}
\usepackage{graphicx}
\usepackage{epstopdf}
\usepackage{algorithm}
\usepackage{algpseudocode}
\usepackage{hyperref}
\ifpdf
  \DeclareGraphicsExtensions{.eps,.pdf,.png,.jpg}
\else
  \DeclareGraphicsExtensions{.eps}
\fi

\newcommand{\E}{\mathbb E}
\DeclareMathOperator{\tr}{trace}
\newsiamremark{remark}{Remark}
\newsiamremark{hypothesis}{Hypothesis}
\crefname{hypothesis}{Hypothesis}{Hypotheses}
\newsiamthm{claim}{Claim}

\headers{Gradient-Free sequential BOED with Interacting Particle Systems}{R. Gruhlke, M. Hanu, C. Schillings and P. Wacker}

\title{Gradient-Free Sequential Bayesian Experimental Design via Interacting Particle Systems}

\author{%
Robert Gruhlke, Matei Hanu, Claudia Schillings\thanks{Freie Universit\"at Berlin, Fachbereich Mathematik und Informatik, 14195 Berlin, Germany 
  (\email{\{r.gruhlke, matei.hanu, c.schillings\}@fu-berlin.de})}
\and Philipp Wacker\thanks{University of Canterbury, Christchurch, New Zealand (\email{philipp.wacker@canterbury.ac.nz})}
}

\usepackage{amsopn}

\makeatletter
\newcommand*{\addFileDependency}[1]{
  \typeout{(#1)}
  \@addtofilelist{#1}
  \IfFileExists{#1}{}{\typeout{No file #1.}}
}
\makeatother

\ifpdf
\hypersetup{
  pdftitle={Gradient-Free Sequential Bayesian Experimental Design via Particle Methods},
  pdfauthor={Robert Gruhlke, Matei Hanu, Claudia Schillings, and Philipp Wacker}
}
\fi

\usepackage{diagbox}
\usepackage{scalerel}
\usepackage{graphicx}
\usepackage{subcaption}
\usepackage{todonotes}
\newtheorem{assumption}[theorem]{Assumption}

\theoremstyle{plain} \theorembodyfont{\rmfamily}

\numberwithin{equation}{section}
\numberwithin{figure}{section}
\numberwithin{table}{section}

\newcommand{\Nens}{{N_{\rm ens}}}
\newcommand{\Jeki}{{J_{\text{\tiny EKI}}}}

\newcommand{\Taldi}{{T_{\text{\tiny ALDI}}}}
\newcommand{\Teki}{{T_{\text{\tiny EKI}}}}

\newcommand{\Nobs}{{N_{\rm obs}}}
\newcommand{\comment}[1]{}

\newcommand{\R}{{\mathbb{R}}}

\newcommand{\N}{{\mathbb{N}}}
\newcommand{\cF}{{\mathcal{F}}}

\newcommand{\argmax}{{\mathrm{argmax}}}

\newcommand{\spann}{{\mathrm{span}}}

\definecolor{dgreen}{RGB}{0,100,0}

\DeclareMathOperator{\e}{\mathrm e}
\newcommand{\ev}[1]{ {\boldsymbol{\mathbb  E}} \left[  #1 \right]}

\renewcommand{\d}{\mathrm{d}}
\DeclareMathOperator{\Cov}{Cov}

\DeclareMathOperator{\Var}{Var}

\newcommand{\vertiii}[1]{{\left\vert\kern-0.25ex\left\vert\kern-0.25ex\left\vert #1 
\right\vert\kern-0.25ex\right\vert\kern-0.25ex\right\vert}}

\usepackage{algorithm}
\usepackage{algpseudocode}

\usepackage{bm}

\usepackage{siunitx} 
\newcommand{\gray}[1]{{\color{gray!50!white}#1}}

\begin{document}

\maketitle

\begin{abstract}
We introduce a gradient-free framework for Bayesian Optimal Experimental Design (BOED) in sequential settings, aimed at complex systems where gradient information is unavailable. Our method combines Ensemble Kalman Inversion (EKI) for design optimization with the Affine-Invariant Langevin Dynamics (ALDI) sampler for efficient posterior sampling—both of which are derivative-free and ensemble-based. To address the computational challenges posed by nested expectations in BOED, we propose variational Gaussian and parametrized Laplace approximations that provide tractable upper and lower bounds on the Expected Information Gain (EIG). These approximations enable scalable utility estimation in high-dimensional spaces and PDE-constrained inverse problems. We demonstrate the performance of our framework through numerical experiments ranging from linear Gaussian models to PDE-based inference tasks, highlighting the method’s robustness, accuracy, and efficiency in information-driven experimental design.
\end{abstract}

\begin{keywords}
 optimal experimental design, Bayesian inverse problems, interacting particle methods, Gaussian approximations
\end{keywords}

\begin{AMS}
  62K05, 62F15, 65C05, 93E10
\end{AMS}

\section{Introduction}

Uncertainty quantification (UQ) is crucial in many fields, enabling informed decision-making in the presence of model, parameter, or data uncertainty. The Bayesian approach to UQ provides a principled framework for incorporating observational data to reduce uncertainty, especially in complex models such as those involving partial differential equations (PDEs). However, acquiring data is often costly or constrained, motivating the need for Bayesian Optimal Experimental Design (BOED), which seeks to maximize the expected information gain from data.

In this work, we propose a gradient-free BOED framework suited for settings where gradients of the forward model are unavailable or expensive to compute. Specifically, we combine Ensemble Kalman Inversion (EKI) as a derivative-free optimizer with ALDI (Affine-invariant Langevin dynamics) as a sampling method in a sequential design setting. To address the high computational cost associated with nested integrals in BOED, we employ Gaussian approximations to efficiently estimate expected utilities. Our approach provides robust, adaptive experimental designs, while avoiding the need for adjoint computations or gradient information.

This framework enables scalable and flexible BOED in high-dimensional or computationally demanding settings, with a particular focus on PDE-based inverse problems. The resulting methodology improves the quality of data used for Bayesian inference, ultimately enhancing the reliability of calibrated models.

\subsection{Literature overview}

BOED provides a principled framework for selecting experiments that maximize the expected information gained about uncertain model parameters. By treating parameters as random variables and updating beliefs using Bayes’ theorem, BOED naturally accounts for prior knowledge and quantifies uncertainty.  Key foundational works in the field include \cite{Cavagnaro2010,Chaloner1995,Lindley1956,Lindley1972,Ryan2016}, and recent comprehensive reviews such as \cite{Huan_Jagalur_Marzouk_2024,Rainforth2024} have summarized current advances and challenges.

A significant direction in BOED research is sequential design or Bayesian adaptive design \cite{Cheng2005,foster2021deepadaptivedesignamortizing, huan2016sequentialbayesianoptimalexperimental,MacKay1992,Zhou2008}. This approach enables dynamic adjustment of experimental configurations based on data collected in earlier stages. While conceptually powerful, sequential BOED poses considerable computational challenges, particularly when dealing with high-dimensional parameter spaces or expensive-to-evaluate forward models such as PDEs.

The most commonly used utility in BOED is the Expected Information Gain (EIG). Estimating the EIG often involves nested Monte Carlo estimators \cite{Tempone, Rainforth2024, Ryan2016}, where the utility is expressed as an expectation over the difference between the log-likelihood and the log-evidence. This requires an outer Monte Carlo loop over the prior and an inner loop to estimate the intractable evidence for each sample. This estimator converges slowly and is computationally intensive. In \cite{Tempone2,Vesa}, quasi Monte-Carlo approaches has been suggested to accelerate the convergence rate. To mitigate the cost of the double-loop estimators, Gaussian approximations to the posterior have become a practical alternative \cite{Chaloner1995,Ryan2016, doi:10.1137/21M1466499}, avoiding the need for full nested integration by using analytical or surrogate-based approximations of the evidence. 
Gaussian and parametric Gaussian approximations can be understood as variational approximations; see the review \cite{Rainforth2024} for further discussion.
For dynamical systems, \cite{huan2016sequentialbayesianoptimalexperimental} formulates the sequential optimal experimental design  problem as a dynamic program and proposes tractable numerical strategies by using backward induction with regression to construct and
refine value function approximations in the dynamic program. 

Beyond nested estimators, measure transport methods have emerged as another promising direction for BOED, particularly in batch settings \cite{Koval_2024,Marzouk2016,Moselhy2012}. These approaches estimate posterior densities via transport maps and then use Monte Carlo to approximate expected utilities. While effective in lower-dimensional spaces, scalability to high-dimensional problems remains a challenge. In \cite{cui2025subspaceacceleratedmeasuretransport}, the authors propose a scalable framework for sequential BOED in high-dimensional Bayesian inverse problems, introducing a derivative-based upper bound on the incremental expected information gain to guide design and enable parameter dimension reduction via likelihood-informed subspaces and transport maps. A few strategies for sequential BOED, including transport maps or sample transformations between design stages, have been reviewed in \cite{Huan_Jagalur_Marzouk_2024,Rainforth2024,Ryan2016} and others.

In high-dimensional inverse problems, sampling from the posterior is central. A widely studied class of methods uses Langevin dynamics, such as the overdamped Langevin equation \cite{ Pavliotis2014}. Gradient-based sampling techniques like MALA (Metropolis-adjusted Langevin Algorithm) \cite{roberts1998optimal,roberts1996exponential} are known for their theoretical convergence guarantees. However, they require gradient evaluations, which may be unavailable or prohibitively expensive in complex models.

To address this, gradient-free Langevin-type samplers have gained attention. One such method is ALDI (Affine-Invariant Interacting Langevin Dynamics) \cite{garbuno2020affine}, which leverages sample covariance preconditioning for affine invariance and robustness in ill-conditioned problems. 
Recent extensions such as LIDL \cite{LIDL} introduce adaptive ensemble enrichment and homotopy-based dynamics to improve efficiency and robustness, especially for complex or multimodal posteriors.
Similar ideas are present in consensus-based sampling \cite{Carrillo2022_Consesns} and Fokker–Planck-based formulations \cite{Sahani2019,Reich_Weissmann_2022}. These methods scale well in high-dimensional spaces and avoid reliance on gradient information.

As an optimization method, the Ensemble Kalman Inversion (EKI) \cite{Iglesias_2013,Schillings2016} has been widely used for solving inverse problems, offering a derivative-free alternative that is especially effective in high dimensions and for PDE-constrained problems. EKI blends the strengths of variational and Bayesian inference and benefits from simple implementation and robustness to small ensemble sizes. Stability and convergence of EKI, particularly in the continuous-time limit, have been studied in \cite{Bloemker2021,Schillings2017}. Regularization, especially via Tikhonov methods, plays a key role in ensuring convergence \cite{Tong2020, Weissmann2022}, with recent developments exploring adaptive and constrained formulations \cite{Albers_2019,Chada2019,Hanu_2024_CC,Herty_2020}.

In computationally intensive settings, EKI has also been combined with subsampling strategies \cite{Hanu2024,Hanu_2023,Latz2021}, inspired by stochastic approximation methods \cite{Robbins1951}, to reduce cost per iteration. 

Our work builds on this foundation by proposing a gradient-free framework for sequential BOED, combining EKI as an optimizer with ALDI as a sampler. Gaussian approximations are employed to avoid nested Monte Carlo loops, yielding a scalable and practical method for high-dimensional experimental design with PDE constraints.

\subsection{Main ideas and contributions of the paper}

We propose a gradient-free BOED approach for maximizing the EIG in a sequential setting. Our method combines EKI as a derivative-free optimizer with the gradient-free ALDI sampler. To estimate the EIG, we employ Gaussian approximations of the joint distribution, as well as a parametrized Laplace approximation. These approximations yield bounds on the true EIG: approximating the posterior results in a lower bound, while approximating the marginal distribution gives an upper bound. In our numerical illustrations, we focus on a Gaussian approximation of the marginal distribution, though we also explore the impact of using a Laplace approximation. A brief overview of the proposed method is given in \cref{algo} and visually summarized in \cref{fig:illustration_algorithm}.

\begin{algorithm}
  \caption{Sequential gradient-free BOED with interacting particle systems}\label{algo}
 
  \begin{algorithmic}[1]
  \Require $\left\{\begin{array}{ll}
            N & \text{length of data assimilation window}, \\
            \mathrm{EIG}\colon \mathcal{D}\times\mathcal{P}_{ac}(\mathbb{R}^d)\to\mathbb{R}& \text{BOED utility on design space $\mathcal{D}$}\\  & \text{and the prior at the current step},\\
            \Jeki & \text{ensemble size for EKI}, \\
            \mu_0 & \text{initial  prior measure in $\mathcal{P}_{ac}(\mathbb{R}^d)$}.
            \end{array}
    \right. $
 \State Initialize design points $\{p_0^{(i)}\}_{i=1}^\Jeki$ for $i=1,\ldots,\Jeki$.
  \For{$n$ in $\{0,\ldots, N-1\}$}
  \gray{\Comment{Find best design integrating the next data point}}
  
    \State Compute the next optimal design $p_{n+1}^\dagger$ via EKI using (parametric) Gaussian approximations of the joint parameter-data distribution to estimate the EIG. 
    
   \State Use Aldi to generate samples 
   from the posterior with observations $y_{n+1}^\dagger$ for the design $p_{n+1}^\dagger$.
  \EndFor
  
  \end{algorithmic}
\end{algorithm}

\begin{figure}
    \centering
    \includegraphics[width=1\linewidth, trim=0 240 0 240, clip]{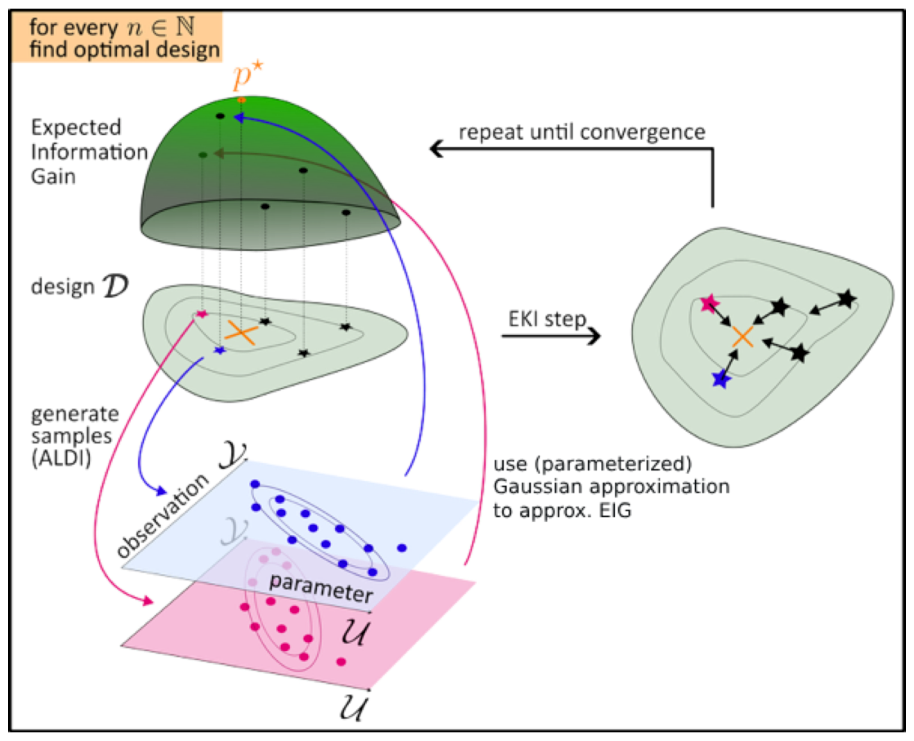}
    \caption{Schematic description of algorithm}
    \label{fig:illustration_algorithm}
\end{figure}

Our main contributions are as follows:

\begin{itemize}
    \item \textit{Gradient-free framework for sequential BOED:}
We propose a novel algorithmic framework for BOED in a sequential setting that does not rely on gradient information of the forward model. This makes our method particularly suitable for complex, black-box models such as PDEs.
    \item \textit{Integration of EKI and ALDI:}
Our approach combines EKI for design optimization with ALDI for posterior sampling. Both components are gradient-free and ensemble-based, enabling efficient application in high-dimensional settings.
    \item \textit{Efficient EIG estimation via variational approximations:}
We propose (para\-metrized) Gaussian approximations to estimate the EIG, yielding computationally efficient upper and lower bounds, ensuring the control on the approximations. The parametrized Laplace approximation is considered to improve the lower bound, if necessary. We analyze the theoretical behavior of these bounds and their asymptotic properties in near-linear regimes and large data/small noise settings.
\item \textit{Numerical validation in increasingly complex scenarios:}
We demonstrate the performance of our method through three numerical experiments of increasing complexity: a linear Gaussian model, a mildly nonlinear example, and a PDE-based inverse problem. These illustrate both the effectiveness and limitations of the proposed framework.
\end{itemize}
We organize the manuscript as follows. In \cref{sec:BOED}, we introduce the BOED framework. \Cref{sec:Gauss_Laplace_approx} quantifies the effects of (parametrized) Gaussian approximations on the EIG. The gradient-free Langevin sampler is presented in \cref{sec:evaluating}, and EKI as the optimizer in \cref{sec:Opt_EKI}. \Cref{sec:algo} summarizes the full algorithm. Numerical experiments are reported in \cref{sec:Numerics}, and we conclude in \cref{sec:conclusions}.

\subsection{Notation}

We define the rescaled norm $\|x\|_{\Sigma} := \langle x,\Sigma^{-1} x\rangle$, $x\in\R^d$ given a symmetric positive definite matrix $\Sigma\in\R^{d\times d}$, where $\langle \cdot, \cdot\rangle$ denotes the euclidean inner product over $\R^d$. Moreover, when we work in a probabilistic setting we consider a probability space $(\Omega,\mathcal{A},\mu)$ and denote by $\mathcal{B}(\mathcal{X})$ the Borel-$\sigma$ Algebra over some space $\mathcal{X}$, i.e. the smallest $\sigma$-Algebra that contains all open sets in $\mathcal{X}$. For a random variable $u$, the corresponding distribution is denoted by $\mu_u$. Furthermore, in case of a finite dimensional setting we denote the Lebesgue density of those random variables as $\pi_u(u)$. 
We denote by $\mathbb{E}_u\left[\cdot\right]$ the expected value with respect to the random variable $u$. Moreover, we denote for a real-valued random variable $u\in L^q=\{u\mid u:\Omega\to\mathbb{R}\, \mu\text{-mb.} \wedge \mathbb{E}[\|u\|^q]<\infty\}$ the $ L^q$ norm as $\mathbb{E}[\|u\|^q]^{1/q}$ . Let $(u_n)_{n\in\mathbb{N}}$ be a series of real-valued random variables with  $\mathbb{E}[\|u_n\|^q]<\infty$ for all $n\in\mathbb{N}$ then we denote $ L^q$ convergence as $u_n \stackrel{ L^q}{\longrightarrow} u $ if $\mathbb{E}[\|u_n-u\|^q]\to 0$ as $n\to\infty$. Additionally, we denote by 
\[
\mathcal{P}_{\mathrm{ac}}(\mathbb{R}^d) := \left\{ \mu \in \mathcal{P}(\mathbb{R}^d) \,\middle|\, \mu \ll \lambda^d \right\}\,,
\] 

the set of all probability measures on $\mathbb{R}^d$ that are absolutely continuous with respect to the Lebesgue measure $\lambda^d$, where 
\[
\mathcal{P}(\mathbb{R}^d) := \left\{ \mu \,\middle|\, \mu \text{ is a probability measure on } \mathbb{R}^d \right\}\,.
\]
Throughout this work we will denote by $u$ the unknown parameter $y$ will denote the data variable, and $p$ the design. In case of a set of given observations we use the notation $\bm{y}_n^\dagger=(y_1^\dagger,\ldots,y_n^\dagger)$ for any $n\in\mathbb{N}$. We will use the $\dagger$-superscript notation to distinguish from variable parameters, e.g. $p$ and $p^\dagger$ and $y$ and $y^\dagger$.

Furthermore, when we consider approximations we denote them via $\widetilde{\cdot}$, e.g. $\widetilde{\pi}(u)$ is an approximation of $\pi(u)$. Finally, for $J\in \mathbb{N}$ we define the empirical means and empirical covariance of a sample of observations $\{y^{(1)}\ldots,y^{(J)}\}$ and parameters $\{u^{(1)}\ldots,u^{(J)}\}$ (or random variables in general) as

\begin{equation*}
    \widetilde{m}_{y}=\frac{1}{J}\sum_{j=1}^J y^{(j)},
    \end{equation*}
    and the corresponding covariances
 \begin{equation*}   
    \widetilde{C}_y =\frac{1}{J}\sum_{j=1}^J (y^{(j)}-\widetilde{m}_{y})\otimes(y^{(j)}-\widetilde{m}_{y}), \quad
    \widetilde{C}_{uy} =\frac{1}{J}\sum_{j=1}^J (u^{(j)}-\widetilde{m}_{u})\otimes(y^{(j)}-\widetilde{m}_{y}),
\end{equation*}
where we omit the explicit dependence on $J$.

We usually also include the design $p$ in the notation when the observations depend on it, e.g. $\widetilde{C}_{y\mid p}$. Moreover, when we consider convergence that depends on $J\to\infty$ we also include the particle count, e.g. $\widetilde{C}_{y\mid p}^{(J)}$. Additionally, when we work with ensemble Kalman inversion the mixed covariance between the particles and the particles under the forward operator will play a crucial role, i.e. given some operator $\mathcal{F}(\cdot)$, we use

\begin{equation*}
    \bar{\mathcal{F}}(u)=\frac{1}{J}\sum_{j=1}^J \mathcal{F}(u^{(j)}), \quad
    \widetilde{C}_{u\mathcal{F}} =\frac{1}{J}\sum_{j=1}^J (u^{(j)}-\widetilde{m}_{u})\otimes( \mathcal{F}(u^{(j)})-\bar{\mathcal{F}}(u))
\end{equation*}

In general we will use subscripts when we denote sequential observations at times $1,\ldots,n$. When we consider samples we use superscripts, here we use the superscript $(j)$ to denote samples that come from a Langevin sampler and $(i)$ to indicate different particles in the EKI setting. When we consider both we use the superscript $(i,j)$.

\section{Sequential Bayesian optimal experimental design} \label{sec:BOED}

Sequential BOED is an adaptive approach to experiment selection that maximizes information gain while incorporating observations sequentially, i.e. sequential BOED iteratively updates beliefs based on newly acquired data, ensuring that each subsequent experiment is optimally chosen. By leveraging Bayesian inference, this method continuously updates the uncertainty in model parameters. We first introduce the parameter estimation problem, i.e. the Bayesian approach to estimating the unknown parameters from sequential data, for a given design, and then discuss the optimization problem for experiment selection.

At a given time step $n\in\mathbb{N}$, for a given observation $y_n\in\mathcal{Y}=\mathbb{R}^k$, 
the goal is to recover the unknown parameters $u\in \mathcal{X}$ based on the model
\begin{equation}\label{eqn:ip_doe}
    y_n=G_n(u,p_n)+\eta_n\,, \quad n\in\mathbb{N}
\end{equation}
where $\mathcal{X}$ is a separable Hilbert space, 
$p_n \in \mathcal{D}$ are design parameters with design space $\mathcal{D}\subseteq \mathbb{R}^{\text{dim}_p}$, $\text{dim}_p\in\mathbb{N}$ and $\eta_n$ is Gaussian additive observational noise from the distribution $\mathcal{N}(0,\Gamma)$, where $\Gamma\in\mathbb{R}^{k\times k}$ symmetric positive definite. Moreover, we consider possibly nonlinear continuous forward maps $G_n: \mathcal{X} \times \mathcal D \rightarrow \mathcal{Y}$.

Reformulating this in the Bayesian setting, we consider the unknown parameters $u$ in \eqref{eqn:ip_doe} to be random variables with absolute continuous prior distribution $\mu_n$ with density $\pi_n$ at time step $n$.
For simplicity, we assume that the parameter space is finite dimensional, i.e. $\mathcal{X}=\R^{d}$ for $d\in\N$, and the prior distribution $\mu_0$ at time $n=0$ is Gaussian as well as independent of the noise $\eta_n$ for all $n\in\mathbb{N}$. Here, for $n>0$ the prior distribution $\mu_n$ is the result of an iterative update within the sequential BOED.

The random variable $y_{n}\mid u, p_n$ is called the \textit{likelihood} and it holds $y_{n}\mid u,p_n \sim\mathcal{N}(G_n(u,p_n),\Gamma)$. For a concrete realization of $y_n^\dagger$ of $y_n$ for a given design $p_n^\dagger$, the goal of computation in the $n$-th step is the distribution of $u\mid y_n^\dagger, p_n^\dagger := u\mid \{y_{n}=y_n^\dagger,p_n=p_n^\dagger\}$, the so-called \textit{posterior distribution}.\\

\begin{remark}
    The inverse problem 
    \begin{equation}\label{eqn:ip_doe_basic}
    y=G(u,p)+\eta\,,
\end{equation}
i.e. a single observation, corresponds to one step in the above setting. We often focus in the following on one step in the filtering problem to illustrate our methodology, but are ultimately interested in the sequential setting.
\end{remark}

So we can interpret the filtering problem as adapting the $n$-th posterior distribution with every new observation $y_{n}^\dagger$ where the latter posterior distribution acts as new prior at time step $n>1$. The density of the posterior distribution can then be written in the form
\begin{equation}\label{eqn:post_dist}
    \pi_{n}(u):=\pi_{u\mid y_{n}^\dagger,p_n^\dagger}(u)=\frac{1}{Z_{n}}\exp(-\Phi(u,y_n^\dagger,p_n^\dagger))\pi_{n-1}(u)\,,
\end{equation}

where $Z_{n}=\int_{\mathbb R^d} \exp(-\Phi(u,y_n^\dagger,p_n^\dagger))\pi_{n-1}(u) \mathrm d u$ denotes the normalization constant, which is assumed to be larger than $0$ and 

\begin{equation}\label{eqn:log_data_density}
    \Phi(u,y,p) = \tfrac{1}{2}\|y-G_n(u,p)\|_\Gamma^2\,, \quad \forall (u,y,p)\in \mathcal{X}\times\mathcal{Y}\times \mathcal{D}
\end{equation}
is the potential. 

Iteratively applying \eqref{eqn:post_dist} on all previous observations $\bm{y}_n^{\dagger}=(y_1^\dagger,\ldots,y_n^\dagger)$ and associated designs $\bm{p}_n^\dagger=(p_1^\dagger,\ldots,p_n^\dagger)$ yields

\begin{equation} \label{eqn:post_dist_seq}
    \pi_{u_n\mid \bm{y}_n^\dagger,\bm{p}_n^\dagger}(u)= \pi_n(u) = \frac{1}{Z_{(n)}}\exp\left(-\sum_{\ell=1}^n\Phi(u,y_\ell^\dagger,p_\ell^\dagger)\right)\pi_0(u)\,,
\end{equation}

where $Z_{(n)}=\int_{\mathbb R^d} \exp\left(-\sum_{\ell=1}^n\Phi(u,y_\ell^\dagger,p_\ell^\dagger)\right)\pi_0(u) \mathrm d u$. \\

\begin{remark}
    Based on \eqref{eqn:post_dist} and \eqref{eqn:post_dist_seq} for notational simplicity in our algorithms and to underline the sequential dependency, we use the shorthand \( \pi_n(u) \) to denote both, the posterior density at time \( n \) and prior density at time $n+1$. The notation $\pi_{u_n\mid \bm{y}_n^\dagger,\bm{p}_n^\dagger}$ shall be used to explicitly state the given observation and designs.
\end{remark}

Let us now focus on the experiment selection. 
For a given family of data distributions $\{\mu_{y_n|p_n}\mid p_n\in \mathcal{D}\}$ with Lebesgue density $y\mapsto \pi_{y\mid p_n}(y)$,
we are interested in solving the following optimization task at every time step $n\in\mathbb{N}$
\begin{equation}
    \sup\limits_{p_n\in \mathcal{D}} \mathcal{L}(p_n)  
    =
    \sup\limits_{p_n\in \mathcal{D}} 
    \mathbb{E}_{y_n\mid p_n}\left[ 
    \,\mathcal{U}(\cdot, p_n) ]\,\right]
    = 
     \sup\limits_{p_n\in \mathcal{D}} 
     \left\{
     \int\limits_{\mathcal{Y}} \mathcal{U}(y,p_n)\pi_{y\mid p_n}(y)\,\mathrm{d}y
     \right\}\,,
     \label{eq:general_optproblem}
\end{equation}
for some user defined utility function, that quantifies the benefit of choosing the design $p_n$ and observing an outcome $y_n$, given the unknown parameters of the model $u$.
\begin{equation}
    \mathcal{U}\colon \mathcal{Y}\times \mathcal{D} \to \mathbb{R}.
\end{equation}

\subsection{Utility function}
In OED, a utility function serves as a measure of how informative or effective an experiment is for a given goal, such as improving parameter estimation, reducing uncertainty, or enhancing predictive power. The central idea is to design experiments that maximize this utility, ensuring that the collected data is as valuable as possible. Different utility functions exist depending on the problem, including variance reduction, criteria based on the Fisher information matrix, and information-theoretic approaches. We refer to \cite{Huan_Jagalur_Marzouk_2024} for a detailed overview on utility functions. One particularly powerful choice is the Expected Information Gain (EIG), which quantifies how much an experiment is expected to reduce uncertainty in a Bayesian framework.

EIG is motivated by the idea that a well-designed experiment should maximize the expected reduction in uncertainty about a quantity of interest, such as model parameters. It is defined as the expected Kullback-Leibler (KL) divergence between the prior and posterior distributions, measuring how much the posterior distribution is expected to deviate from the prior after observing new data

\begin{equation}
    \mathcal{U}(y, p):= 
    \mathrm{D}_{\mathrm{KL}}(\mu_{u \mid y,p}\,\mid\mid\,\mu_0) = 
    \int\limits_{\mathcal{X}} \log\left(\frac{\pi_{u\mid y,p}(u)}{\pi_0(u)}\right)\pi_{u\mid y,p}(u)\,\mathrm{d}u\,. 
\end{equation}

We denote the objective function in \eqref{eq:general_optproblem} for this utility function as $\mathcal{L}_D(p)$, which is known as the \textit{expected information gain (EIG)}.\\

More specifically we consider the EIG for an arbitrary measure on $\mathbb{R}^d$ that is absolutely continuous w.r.t.\ Lebesgue measure $\lambda$ on $\mathbb{R}^d$.

We define the EIG as a mapping $\text{EIG}: \mathcal{D} \times \mathcal{P}_{\mathrm{ac}}(\mathbb{R}^d)\to\mathbb{R}$ where

\begin{align}
\text{EIG}(p,\mu)&:=\mathbb{E}_{y\mid p}\left[D_{KL}(\mu_{u\mid y,p}\,\mid\mid\,\mu)\right]\notag\\
    &\phantom{:}=\int_\mathcal{Y}\int_\mathcal{X}\log\left(\frac{\pi_{u\mid y,p}(u)}{\pi(u)}\right)\pi_{u\mid y,p}(u)\pi_{y\mid p}(y)\,\mathrm{d}u\,\mathrm{d}y\notag\\
    &\phantom{:}=\int_\mathcal{Y}\int_\mathcal{X}\log\left(\frac{\pi_{y\mid u,p}(y)}{\pi_{y\mid p}(y)}\right)\pi_{u,y\mid p}(u,y)\,\mathrm{d}u\,\mathrm{d}y\,\label{eqn:eig_used}
\end{align}

where $\pi(u)$ denotes the Lebesgue density of $\mu$. For our application $\mu$ will correspond to the prior distribution $\mu_n$ at time step $n$ and $\mu_{u\mid y,p}$ the associated posterior distribution for given $y$ and $p$. For notational simplicity, we will sometimes write \(\text{EIG}(p, \pi)\). 

\begin{remark}
In the non-sequential setting, we typically assume \(\mu = \mu_0\), where \(\mu_0\) is a Gaussian prior, and we denote the EIG simply as \(\text{EIG}(p)\) as familar in the BOED literature.
\end{remark}

The exact evaluation of \eqref{eqn:eig_used} is usually not possible, due to not having a closed formula for the posterior density $\pi_{u\mid y,p}(u)$ (or equivalently the normalization constant $\pi_{y\mid p}(y)$) and approximating the EIG is computationally demanding due to the need to evaluate two nested integrals: one over the possible observations and another over the posterior distribution of the parameters. 

\subsection{Methodological framework for derivative-free sequential BOED}

In the following, we present an overview of our proposed methodology, which provides a fully derivative-free approach to OED by leveraging derivative-free Langevin sampling for joint distribution approximation and Ensemble Kalman Inversion (EKI) for experimental optimization.

The primary challenge lies in efficiently sampling from the joint distribution, particularly in a sequential setting, where the observations are obtained one at a time, and the posterior is updated dynamically.

To address this, we propose using derivative-free Langevin methods to sample from the joint distribution. Traditional Langevin-based samplers leverage gradient information to explore the posterior efficiently, but in many practical cases, computing gradients is either infeasible or expensive. A derivative-free variant allows us to approximate the joint distribution without relying on explicit gradient computations, making it particularly useful when dealing with black-box likelihoods, i.e. for computationally expensive black-box models.

Once samples from the joint distribution are obtained, a key remaining challenge is approximating either the normalization constant or the posterior distribution itself. A natural and computationally efficient approach is to use a Gaussian approximation, which can serve as a surrogate for the true posterior. This provides a tractable way to estimate the EIG while maintaining computational efficiency. In cases where the Gaussian approximation is insufficiently accurate, it can still be used as an initialization for more refined sampling techniques, leading to a practical and scalable approach for experimental design in complex models.

To determine the optimal experimental setup at each time step, we use EKI as a derivative-free optimization method. EKI is particularly useful in high-dimensional and non-convex settings, where efficient tools for gradient computations are not available. By using EKI, we iteratively refine the experimental design without relying on explicit gradient computations. This ensures that the entire approach—from sampling the joint distribution to optimizing the experiment—remains fully derivative-free, making it a practical and scalable framework for sequential Bayesian experimental design in complex models.

In the following sections, we provide a detailed explanation of each individual component of our proposed methodology, starting with Gaussian approximations of the posterior and normalization constants, followed by a discussion of derivative-free Langevin sampling for joint distribution approximation and then the use of EKI for optimizing the experimental design.

\section{Gaussian approximations for efficient EIG estimation}\label{sec:Gauss_Laplace_approx}

Gaussian approximations for EIG are particularly valuable due to their analytical tractability and computational efficiency. Before delving into the details of our proposed approximation, we first review the well-known upper and lower bounds on EIG that arise from approximating either the posterior distribution or the normalization constant. Please see \cite{barber2003im, foster2019variational, Huan_Jagalur_Marzouk_2024, Poole2019} for more details. 
We denote the approximation of $\pi_{y\mid p}(y)$ by $\tilde{\pi}_{y\mid p}(y)$ and the approximation of the posterior distribution $\pi_{u\mid y,p}(u)$ by $\tilde{\pi}_{u\mid y,p}(u)$ for any $y$ and $p$. Then, for an upper bound the following holds
\begin{align}
    \text{EIG}(p)
    &=\mathbb{E}_{u,y\mid p}\left[\log\left(\frac{\pi_{y\mid u,p}(y)\tilde{\pi}_{y\mid p}(y)}{\tilde{\pi}_{y\mid p}(y)\pi_{y\mid p}(y)}\right)\right]\notag\\
    &=\mathbb{E}_{u,y\mid p}\left[\log\left(\frac{\pi_{y\mid u,p}(y)}{\tilde{\pi}_{y\mid p}(y)}\right)\right] - D_{\text{KL}}\left(\mu_{y\mid p}\mid\mid\tilde\mu_{y\mid p}\right)\notag\\
    &\leq \mathbb{E}_{u,y\mid p}\left[\log\left(\frac{\pi_{y\mid u,p}(y)}{\tilde{\pi}_{y\mid p}(y)}\right)\right]    =\mathbb{E}_u\left[D_{\text{KL}}\left(\mu_{y\mid u,p}\mid\mid\tilde\mu_{y\mid p}\right)\right]\,.
    \label{eqn:eig_upper_bound}
\end{align}
For a lower bound the posterior density gets approximated 

\begin{align}
\text{EIG}(p) &= \mathbb{E}_{u,y\mid p} \left[ \log \frac{\tilde{\pi}_{u \mid y, p}(u) \, \pi_{u \mid y, p}(u)}{\pi_0(u) \, \tilde{\pi}_{u \mid y, p}(u)} \right] \notag \\
&= \mathbb{E}_{u,y\mid p} \left[ \log \frac{\tilde{\pi}_{u \mid y, p}(u)}{\pi_0(u)} \right] + \mathbb{E}_{y \mid p} \left[ D_{\text{KL}}(\mu_{u \mid y, p} \mid\mid \tilde\mu_{u \mid y, p}) \right] \notag \\
&\geq \mathbb{E}_{u,y\mid p} \left[ \log \frac{\tilde{\pi}_{u \mid y, p}(u)}{\pi_0(u)} \right] \,.\label{eqn:eig_lower_bound} 
\end{align}

\begin{remark}
    We note that in the sequential setting the upper and lower bounds are given by
    \begin{equation}
        \mathbb{E}_{u,y\mid p} \left[ \log \frac{\tilde{\pi}_{u \mid y, p}^n(u)}{\pi_{n-1}(u)} \right] \leq\text{EIG}(p, \pi_{n-1})\leq \mathbb{E}_{u,y\mid p}\left[\log\left(\frac{\pi_{y\mid u,p}(y)}{\tilde{\pi}_{y\mid p}^n(y)}\right)\right]\,.
    \end{equation}

Furthermore, we note that the nominator in the lower bound satisfies
for $n>1$
    \begin{align}
        \pi_{n-1}(u)&=\frac{1}{Z_{n-1}}\exp\left(-\tfrac{1}{2}\|y_{n-1}-G_{n-1}(u,p_{n-1})\|_\Gamma^2\right)\pi_{n-2}(u)\label{eqn:prior_nminus1_onestep}\\
        &= \frac{1}{\prod_{\ell=1}^{n-1} Z_{l}}\exp\left(-\tfrac{1}{2}\sum_{\ell=1}^{n-1}\|y_{\ell}-G_{\ell}(u,p_{\ell})\|_\Gamma^2\right)\pi_0(u)\,\label{eqn:prior_nminus1_nstep}
    \end{align}
        
        where $Z_{l}$ denotes the normalization constant of $\pi_l$.\\

\end{remark}

We will first focus on the lower bound, i.e. on approximations of the posterior distributions.
When the observed data is highly informative, either due to a large amount of data or low noise levels, the posterior distribution tends to concentrate around its mode. By employing a Laplace approximation for the posterior, the distribution is approximated as a Gaussian centered at the maximum a posteriori estimate, effectively capturing local curvature information. We will demonstrate that, under suitable assumptions, the bound based on the Laplace approximation becomes tight in the small noise or large data limit. Additionally, we will explore the use of a Gaussian sample-based approximation for the joint distribution, enabling efficient estimation of both the posterior and the normalization constant, which in turn facilitates scalable computation of EIG bounds even in high-dimensional settings. Finally, we will discuss the complementary strengths of these two approaches and how their combination can further enhance the accuracy and efficiency of EIG estimation.\\

\subsection{Laplace approximation}
In this subsection, we analyze the Laplace approximation in the small-noise or large-data regime. Specifically, we focus on a single step in the sequential setting, assuming that the data is highly informative—either due to low noise levels or the availability of a large set of observations at each time step. The limit we consider here characterizes the asymptotic behavior and provides justification for applying the Laplace approximation even in cases where the noise is finite but small, or when a large amount of data is available. We consider concentrating posterior, scaled by a parameter $\varepsilon$, of the following form
\begin{equation}\label{equ:post}
	\nu_{\varepsilon}(\d u) = \frac 1{Z_{\varepsilon}} \exp\left(-\frac1{\varepsilon} \Phi_{\varepsilon}(u)\right) \nu_u(\d u),
	\quad 
	Z_{\varepsilon} := \int_{\R^d} \exp\left(-\frac1{\varepsilon} \Phi_{\varepsilon}(u)\right) \nu_u(\d u),
	\quad
	{{\varepsilon}\ll 1},
\end{equation}
In the large data setting (not to be confused with the sequential data setting), we can interpret $\varepsilon=\frac{1}{M}, M\in \mathbb N$ and consider potentials
\begin{equation*}
    \Phi_M(u)=\tfrac{1}{M}\sum_{l=1}^M \Phi(u,y^{(l)},p)\,.
    \end{equation*}
Several recent publications have already proposed the Laplace approximation as an effective method for directly approximating posterior distributions in BOED (see e.g. \cite{Cavagnaro2010,Lewi2009,Long2013,Long2022,Ryan2015,Ghattas-Laplace}). Parameterizing the Laplace approximation for each observations point has been for example discussed in \cite{Ghattas-Laplace}. In the following, we build the theory for this approach, in particular we derive results related to the Kullback-Leibler divergence, which then leads to a lower bound of the EIG. 
This paper extends the analysis in \cite{Laplace} by examining the distance of the posterior to the Laplace approximation in the KL divergence for $\varepsilon\to0$.

We define
\begin{equation*}
    I_{\varepsilon}(u) = \Phi_{\varepsilon}(u) - \varepsilon\log(\pi_u)\,. 
\end{equation*}

\begin{assumption}\label{assum:LA_0}
There holds $\Phi_{\varepsilon}, \pi_u \in C^3(\mathbb R^d,\mathbb{R})$.
Furthermore, $I_{\varepsilon}$ has a unique minimizer $u_{\varepsilon} \in \mathbb R^d$, satisfying
\[
	I_{\varepsilon}(u_{\varepsilon}) = 0,
	\quad
	\nabla I_{\varepsilon}(u_{\varepsilon}) = 0,
	\quad
	\nabla^2 I_{\varepsilon}(u_{\varepsilon}) > 0,
\]
where the latter denotes positive definiteness.
For all $\varepsilon$
\begin{enumerate}
\item 
there exist the limits
\begin{equation}\label{equ:x_star}
	u_\star := \lim_{{\varepsilon}\to 0} u_{\varepsilon}
	\qquad
	H_\star := \lim_{{\varepsilon}\to 0} H_{\varepsilon},
	\qquad
	H_{\varepsilon} := \nabla^2\Phi_{\varepsilon}(u_{\varepsilon})\,,
\end{equation}
in $\mathbb{R}^d$ and $\R^{d \times d}$, respectively, with $H_\star$ being positive definite.

\item
For each $r > 0$ there exists an $\varepsilon_r>0$, $\delta_r>0$ and $K_r < \infty$ such that
\[
	\delta_r \leq \inf_{x \notin B_r(u_{\varepsilon})\cap \mathbb R^d} I_{\varepsilon}(x) \qquad \forall 0<\varepsilon\leq \varepsilon_r\,,
\]
as well as 
\[
	\max_{x\in B_r(0) \cap \mathbb R^d} \|\nabla^3 \log \pi_u(x)\| \leq K_r,
	\max_{x\in B_r(0)\cap \mathbb R^d} \|\nabla^3 \Phi_{\varepsilon}(x)\| \leq K_r \quad \forall 0<\varepsilon\leq \varepsilon_r\,.
\]

\end{enumerate}
\end{assumption}

Given \cref{assum:LA_0} we define the \textit{Laplace approximation} of the concentrating distributions $ \nu_{u_{\varepsilon}\mid y^{({\varepsilon})},p} $ as the following Gaussian measure

\[
\mathcal{L}_{{u_{\varepsilon}\mid y^{({\varepsilon})},p}} := \mathcal{N} (u_{\varepsilon}, \varepsilon C_{\varepsilon}), \quad C_{\varepsilon}^{-1} := \nabla^2 I_{\varepsilon} (u_{\varepsilon})\,.
\]

Thus, the unnormalized density is given by

\[
\tilde\pi_{\varepsilon}(x) = \exp \left( - \frac{1}{2\varepsilon} \| x - u_{\varepsilon} \|^2_{C_{\varepsilon}^{-1}} \right)\,,
\]

and the normalization constant is  
\[\tilde{Z}_{\varepsilon} := \varepsilon^{d/2} \sqrt{\det(2\pi C_{\varepsilon})}\,.
\]

The motivation for the Laplace approximation comes from a second order Taylor approximation of $I_{\varepsilon}$ around $u_{\varepsilon}$. We obtain
\begin{align*}
    I_{\varepsilon}(x) &\approx  I_{\varepsilon}(u_{\varepsilon})+\nabla I_{\varepsilon}(u_{\varepsilon})^\top(x-u_{\varepsilon}) + \frac{1}{2} \|x-u_{\varepsilon}\|^2_{C_{\varepsilon}^{-1}}\\
    &=\frac{1}{2} \|x-u_{\varepsilon}\|^2_{C_{\varepsilon}^{-1}}\,,
\end{align*}
since by assumptions the first two terms are zero. And thus the following approximation can be justified
\[
\nu_{u_{\varepsilon}\mid y^{({\varepsilon})},p}\approx\mathcal{L}_{{u_{\varepsilon}\mid y^{({\varepsilon})},p}}\,.\]

\begin{theorem}
Let \cref{assum:LA_0} be satisfied. Then, 
\begin{equation}
D_{\text{KL}}(\nu_{u_{\varepsilon}\mid y^{({\varepsilon})},p} \mid\mid \mathcal L_{{u_{\varepsilon}\mid y^{({\varepsilon})},p}})\in\mathcal{O}({\varepsilon}^{1/2})\,.
\end{equation}

\end{theorem}

\begin{proof}
    See \cref{thm:laplace_seq_appendix}.
\end{proof}

The previous result focuses on the Kullback-Leibler convergence of the posterior and Laplace approximation for a given instance of the observations. Further assuming that the constant $K_r=K_r(y) \in L^1(\mu_{y|p})$ in \cref{assum:LA_0} leads to
\begin{align}
\mathbb{E}_{y \mid p} \left[ D_{\text{KL}}(\nu_{u_{\varepsilon}\mid y^{({\varepsilon})},p} \mid\mid \mathcal L_{{u_{\varepsilon}\mid y^{({\varepsilon})},p}})\right] \in \mathcal O({\varepsilon}^{\frac12})\,.
\end{align}

We have demonstrated that a parameterized Laplace approximation (with respect to the observations) converges to the posterior in the large-data or small-noise limit under suitable assumptions. Consequently, the bound on the EIG becomes tight in this regime. However, applying this result to design optimization requires computing the Laplace approximation for all possible observations corresponding to each design, which may be computationally prohibitive.

To address this, we propose a much simpler alternative: a Gaussian approximation of the joint distribution, which yields both upper and lower bounds on the EIG. If this approximation lacks sufficient accuracy to provide useful design directions, it serves as an initialization for the parameterized Laplace approximation. This hybrid approach significantly enhances computational efficiency compared to directly employing the Laplace approximation from the outset.\\

\subsection{Gaussian approximation of the joint distribution}
We will find the Gaussian approximation by using samples $(u^{(j)},y^{(j)})_{j=1,\ldots,J}$ from the true joint distribution $\mu_{u,y \mid p}$. Moreover, as, before, we consider a fixed time $n\in\mathbb{N}$ and neglect the time index for the moment.
We define by $\mathcal N(\widetilde m_{u,y \mid p}^{(J)}, \widetilde C_{u,y\mid p}^{(J)})$ with 
\begin{equation*}
    \widetilde m_{u,y\mid p}^{(J)}=\begin{pmatrix}\widetilde m_{u\mid p}^{(J)} \\ \widetilde m_{y\mid p}^{(J)}\end{pmatrix},\quad \widetilde  C_{u,y\mid p}^{(J)}\begin{pmatrix} \widetilde C_{u\mid p}^{(J)} \widetilde C_{uy\mid p}^{(J)}\\\widetilde C_{yu\mid p}^{(J)} \widetilde C_{y\mid p}^{(J)}\end{pmatrix},
\end{equation*} where  $\widetilde m_{u\mid p}^{(J)}\in\mathbb{R}^d, \widetilde m_{y\mid p}^{(J)}\in\mathbb{R}^k,  \widetilde C_{u\mid p}^{(J)}\in\mathbb{R}^{d\times d}, \widetilde C_{uy\mid p}^{(J)}\in\mathbb{R}^{d\times k}, \widetilde C_{yu\mid p}^{(J)}\in\mathbb{R}^{k\times d}, \widetilde C_{y\mid p}^{(J)}\in\mathbb{R}^{k\times k}$ 
are empirical means and covariances, the Gaussian approximation of the joint measure $\mu_{u,y|p}$. 

We then define the approximations of the joint, 

\begin{equation}
    \tilde{\mu}_{u, y \mid p}^{(J)}=\mathcal{N}\left( \widetilde m_{u,y\mid p}^{(J)},
\widetilde  C_{u,y\mid p}^{(J)}\right)\label{eqn:gaussapprox_joint}
\end{equation}
the posterior
\begin{equation}
    \tilde{\mu}_{u \mid y, p}^{(J)}=\mathcal{N}(\widetilde m_{u\mid y, p}^{(J)},C_{u\mid y,p}^{(J)})\,,\label{eqn:gaussapprox_posterior}
    \end{equation}

    where 
    \begin{align*}
        \widetilde m_{u\mid y, p}^{(J)}&= \widetilde m_{u\mid p}^{(J)}+\widetilde C_{uy\mid p}^{(J)} (\widetilde C_{y\mid p}^{(J)})^{-1}(y-\widetilde m_{y\mid p}^{(J)})\\
        \widetilde C_{u\mid y,p}^{(J)}&=\widetilde C_{u\mid p}^{(J)}+\widetilde C_{uy\mid p}^{(J)}(\widetilde C_{y\mid p}^{(J)})^{-1}(\widetilde C_{uy\mid p}^{(J)})^\top)\,.
    \end{align*}
And the marginal distribution of $y\mid p$ 
\begin{equation}
    \tilde{\mu}_{y \mid p}^{(J)}=\mathcal{N}(\widetilde m_{y\mid p}^{(J)},\widetilde C_{y\mid p}^{(J)})\,.\label{eqn:gaussapprox_normconst}
\end{equation}
Through the approximation of the posterior distribution $u\mid y,p$ by $\tilde{\mu}_{u \mid y, p}^{(J)}(u)$ we obtain a lower bound of the EIG. On the other hand, by approximating the marginal distribution $y\mid p$ by $\tilde{\mu}_{y \mid p}^{(J)}(y)$, we obtain an upper bound of the EIG.

In the linear case, the bounds get tight with increasing number of samples.

\begin{lemma}\label{lem:eig_upper_lower_bound_linear}
   Let the forward operator $G:\mathcal{X}\times\mathcal{D}\to\mathcal{Y}$ be linear with respect to the unknown parameter, i.e. $G(u,p) = A(p)u$ for all $p\in\mathcal{D}$, where $A(p)\in\mathbb{R}^{k\times d}$ and consider the approximations $\tilde{\mu}_{u \mid y, p}^{(J)}$ and $\tilde{\mu}_{y \mid p}^{(J)}$ from \eqref{eqn:gaussapprox_posterior} and \eqref{eqn:gaussapprox_normconst}.
   Given $(u_i,y_i)_{i=1,\ldots,J}$ i.i.d. samples from the true joint distribution $\mu_{u,y \mid p}$, then there holds
   \begin{align*}
       D_{\text{KL}}\left(\mu_{y\mid p}\mid\mid\tilde{\mu}_{y\mid p}^{(J)}\right) &\xrightarrow[\mathcal O(J^{-1})]{L^2} 0\\
       D_{\text{KL}}\left(\mu_{u \mid y, p} \mid\mid \tilde{\mu}_{u \mid y, p}^{(J)}\right)&\xrightarrow[\mathcal O(J^{-1})]{L^2} 0\,.
   \end{align*}
\end{lemma}

\begin{proof}

Follows from \cref{cor:samples_gauss}.

\end{proof}

Since the forward operator is not fully linear -- otherwise, the problem could be solved analytically -- we aim to analyze the KL divergence between the proposed marginal distribution  $\tilde\mu_{y\mid p}$  and the true marginal distribution $\mu_{y\mid p}$ as well as the posterior $\mu_{u\mid y,p}$ and its approximation $\tilde\mu_{u\mid y,p}$ for a forward operator that is almost linear. To proceed, we make the following assumption:

\begin{assumption}\label{ass:linear_apprx_forward_operator}
    For a fixed design $p\in\mathcal{D}$ we assume that the forward operator $G\in C^2(\mathcal{X}\times \mathcal{D},\mathbb{R}^k)$ is of the form
    \begin{equation}
        G(u,p)=A(p)u+\tau F(u,p)\,,
    \end{equation}
    for $A(p)\in\mathbb R^{k\times d}$, $\tau\in\mathbb R$ and nonlinear perturbation $F\in \mathcal C^2(\mathcal{X}\times \mathcal{D},\mathbb R^k)$.

\end{assumption}

Given this forward operator, the likelihood term has the form

\begin{equation*}
    \exp(-\tfrac{1}{2}\|y-G(u,p)\|^2_\Gamma)=\exp(-\tfrac{1}{2}\|y-A(p)u\|^2_\Gamma)\exp(-\tfrac{1}{2} \text{Res}(\tau,F,u,y,p))\,,
\end{equation*}

where
\[\text{Res}(\tau,F,u,y,p)=\|\tau F(u,p)\|^2_\Gamma-2(y-A(p)u)^\top\Gamma^{-1}\tau F(u,p)\,.\]

Thus the data density is given by

\[\pi_{y\mid u,p}(y)\propto\exp(-\tfrac{1}{2}\|y-A(p)u\|^2_\Gamma)\exp(-\tfrac{1}{2} \text{Res}(\tau,F,u,y,p))\,.\] 

and thus the marginal satisfies

\begin{align}
   \pi_{y\mid p}(y)&=\int\limits_{\mathcal{X}}\pi_{y\mid u,p}(y)\pi_u(u)\mathrm{d}u\notag\\
   &\propto\int\limits_{\mathcal{X}}\exp(-\tfrac{1}{2} \text{Res}(\tau,F,u,y,p)) \exp\left(-\tfrac{1}{2} \|y - A(p)u\|_{\Gamma}^2-\tfrac{1}{2}\|u-m_0\|_{\Sigma_0}^2\right)\mathrm{d}u.\notag
\end{align}

For the Gaussian approximation we first note that

\begin{align*}
    \mathbb{E}_{{y|p}}[y]&=\mathbb{E}_{u}[A(p)u] + \tau \mathbb{E}_u[F(u,p)]=A(p)m_0 + \tau \mathbb{E}_{u}[F(u,p)]\\
    \Cov_{y|p}[y] &= \mathbb{E}_{u,\eta}[(A(p)u+\tau F(u,p) + \eta - \mathbb{E}_{{y|p}}[y])(A(p)u+\tau F(u,p) + \eta - \mathbb{E}_{{y|p}}[y])^\top]\\
    &=A(p)\Sigma_0 A(p)^\top + \Gamma \\
    &+\tau (A(p)\Cov_{u}(u, F(u,p)) + \Cov_{u}( F(u,p),u) A(p)^\top) + \tau^2 \Cov_{u}(F(u,p))\,.
\end{align*}
Assuming that the nonlinear term $F(u,p)$ has bounded second moments uniformly in $p$, we observe that the Gaussian $\tilde{\mu}_{y\mid p}=\mathcal N(\mathbb{E}_{{y|p}}[y], \Cov_{y|p}[y])$ approximation gets incresaingly accurate for smaller values of $\tau$, i.e.
\[
D_{\text{KL}}\left(\mu_{y\mid p}\mid\mid\tilde{\mu}_{y\mid p}\right)\to 0 \qquad \tau\to 0\,.
\]

Similarly, the lower bound \eqref{eqn:eig_lower_bound} increases as the posterior approximation
\[D_{\text{KL}}\left(\mu_{u \mid y, p} \mid\mid {\tilde{\mu}_{u \mid y, p}^{(J)}}\right)\,,\]
becomes more accurate, i.e., as the KL divergence decreases. Assuming that the posterior covariance $\widetilde C_{u\mid y,p}$ is positive definite, there holds that the Gaussian approximation $\mathcal N(\widetilde m_{u\mid y,p},\widetilde C_{u\mid y,p})$ is the minimum of
\[D_{\text{KL}}\left(\mu_{u \mid y, p} \mid\mid {\tilde{\mu}_{u \mid y, p}^{(J)}}\right)= -\mathbb E_{u \mid y, p}[\log \rho]+\mathbb E_{u \mid y, p}[\log \pi_{u \mid y, p}]
\]
for all Gaussian approximations $\rho$. As \(\tau \to 0\), if \(\mu_{u \mid y, p}\) converges to a Gaussian measure, the lower bound approaches tightness, modulo sampling error.

\section{Efficient sampling from the posterior distribution in the sequential setting} \label{sec:evaluating}
As previously noted, evaluating the EIG is generally intractable due to the unknown posterior and joint distributions. We therefore estimate it using the approximation $\widetilde{\mathrm{EIG}}\colon \mathcal{D} \times \mathcal{P}_{\mathrm{ac}}(\mathbb{R}^d)\to\mathbb{R}$

\begin{equation} \label{eqn:eig_approx}
\widetilde{\mathrm{EIG}}(p,\pi):=\frac{1}{J}\sum_{j=1}^J\log\left(\frac{\pi_{y\mid u^{(j)},p}(y^{(j)})}{\pi_{y\mid p}(y^{(j)})}\right)=\frac{1}{J}\sum_{j=1}^J\log\left(\frac{\pi_{u\mid y^{(j)},p}(u^{(j)})}{\pi(u^{(j)})}\right) \approx \mathrm{EIG}(p)\,,
\end{equation}
where $(u_j,y_j)_{j=1}^J$ are samples from the joint distribution. As mentioned above the marginal density $\pi_{y|p}$ (and thus the posterior distribution $\pi_{u|y,p}$) are usually not known and need to be approximated. Approximating $\pi_{y|p}$ yields to $\widetilde{\mathrm{EIG}}(\pi,p)$ being a upper bound, approximating $\pi_{u|y,p}$ gives us a lower bound.\\ In the sequential BOED setting, this approximation requires sampling from the posterior distribution after the first step, as the prior is adapted according to the observations. For this purpose, we introduce Langevin-type methods.

We begin with gradient-based approaches, focusing on Langevin dynamics and its adaptations to gradient-free settings. To facilitate this, we denote the regularized potential in the $n$-th step of the posterior \eqref{eqn:post_dist}
with

\begin{align}
    I(u)&=\Phi(u,y_n,p_n)-\log(\pi_{n-1}(u))\notag\\
    &=\Phi(u,y_n,p_n) - \log\left(Z_{n-1}^{-1}\exp(-\Phi(u,y_{n-1},p_{n-1}))\pi_{n-1}(u))\right)\notag\\
    &=\sum_{\ell=1}^n \Phi(u,y_\ell,p_\ell)+\sum_{\ell=1}^{n-1}\log(Z_\ell)-\log(\pi_{0}(u))\,. \label{eqn:potential_joint}
\end{align}

While we initially consider gradient-based techniques, our primary focus is on gradient-free samplers due to their greater flexibility in settings where gradients are not accessible.

\subsection{Enhanced Langevin methods}

Given a \textit{single-particle} starting point $u_0 \in \mathbb{R}^d$ the first order overdamped Langevin process is of the form
\begin{equation}
    \mathrm{d}u_t = -\nabla I(u_t)\mathrm{d}t + \sqrt{2}\mathrm{d}W_t\,, \label{eq:basic_SDE}
\end{equation}
where $W_t$ is $d$-dimensional Brownian motion.

Collecting $J\in\mathbb{N}$ particles $\{U_t^{(j)}\}_{j=1}^{dJ}$ at time $t$ into a vector
\begin{equation*}
    U_t = \operatorname{vec}(u_t^{(1)},u_t^{(2)},\ldots, u_t^{(J)}) \in \mathbb{R}^{d},\quad t\geq 0\,,
\end{equation*}
many interacting particle system approaches admit the general form
\begin{equation}\label{eq:general_SDE}
    \mathrm{d}u^{(j)}_t = -A(U_t)\nabla_{u^{(j)}_t}\mathcal{V}(Z_t)\mathrm{d}t + \Gamma(U_t)\mathrm{d}W^{(j)}_t\qquad\text{for } j=1,\ldots,J\,.
\end{equation}
Here, $A(U_t)\in \mathbb{R}^{d\times d}$, $\Gamma(U_t)\in\mathbb{R}^{d}$, and $W^{(j)}_t$ are independent $d$-dimensional Brownian motions for $j=1,\ldots,J$ and $\mathcal{V}\colon\mathbb{R}^{d}\rightarrow \mathbb{R}$ usually depends on the potential $I$.

Choosing $A \equiv I_d$, $\Gamma \equiv \sqrt{2}I_d$ and $\mathcal{V}(U_t) = \sum_{j=1}^{J} I(u_t^{(j)})$ in \eqref{eq:general_SDE} leads to a particle system where each particle follows the process \eqref{eq:basic_SDE} independently. 
First, there is no interaction between the particles, which can limit sampling efficiency. Second, the system lacks affine invariance, meaning its convergence behavior is not preserved under linear affine transformations of the state variables. We will discuss in the following a method, which overcomes the drawbacks and, in addition, has a gradient-free extension. 

\subsection{\texorpdfstring{Affine invariant Langevin dynamics (ALDI, cf.~\cite{garbuno2020affine})}{Affine Invariant Langevin Dynamics (ALDI)}}

Standard Lan\-gevin dynamics can be highly sensitive to the geometry of the target distribution. If the posterior is stretched or skewed (e.g., due to parameter correlations or differing scales), vanilla Langevin methods may mix slowly and require careful tuning of step sizes. Affine invariant methods overcome this by adapting to the local geometry of the target, ensuring that the algorithm performs similarly regardless of how the problem is scaled or rotated. This makes them especially effective for sampling from poorly conditioned or anisotropic distributions without requiring extensive manual tuning. We start the discussion by a gradient based version.
\subsubsection{Gradient-based ALDI}

The iteration rule of gradient based ALDI incorporates interaction of the particles in the iteration rule, the update formula is given by

\begin{equation}\label{eq:aldi}
    \mathrm{d}u^{(j)}_t = -(\widetilde{C}_u)_t\nabla I(u_t^{(j)})\mathrm{d}t + \frac{d +1}{J}(u_t^{(j)}-\overline{u}_t)\mathrm{d}t + \sqrt{2}(\widetilde{C}_u)_t^{1/2}\mathrm{d}W_t^{(j)}\,.
\end{equation}

Under strong growth bound conditions on $I$, $\nabla I$ and $\operatorname{Hess}I$, and given $J>d+1$, ALDI is affine invariant and ergodic, i.e. convergence to the target distribution in total variation distance can be proven \cite{garbuno2020affine}.

In practice, ALDI is used with a non-symmetric generalization of the square root
\begin{equation}\label{eq:nonsym_sqrt}
    (\widetilde{C}_u)_t^{1/2} = \dfrac{1}{\sqrt{J}}\left( u_t^{(1)} - \overline{u}_t,\ldots, u_t^{(J)} - \overline{u}_t \right) \in \mathbb{R}^{D \times J}\,,
\end{equation}
such that $\widetilde{C}_u = \widetilde{C}_u^{1/2}(\widetilde{C}_u^{1/2})^{T}$, which can be 
obtained without additional computational cost.

\subsubsection{Gradient-free ALDI}

When gradients are unavailable or expensive to compute, gradient-free ALDI enables efficient sampling by combining affine invariance with gradient-free updates. We refer to \cite{garbuno2020affine} for more details. 
By Taylor's theorem we can justify the following (it holds exactly in the case of a linear forward operator)

\begin{align*}
    -(\widetilde{C}_u)_t\nabla I(u_t^{(j)})
    &\approx  
    -\tilde{\mathcal{A}}(U_t)\\
    -\tilde{\mathcal{A}}(U_t)&:=
    (\widetilde{C}_{u,\mathcal{G}})_t\Gamma^{-1}\left(y-G(u_t^{(j)},p)\right) - (\widetilde{C}_u)_t\Sigma_0^{-1}(u_t^{(j)} - m_0)\,,
\end{align*}

where $m_0\in\mathbb{R}^d $ and $\Sigma_0\in\mathbb{R}^{d \times d}$ denote mean and covariance of a Gaussian prior.\\

Then gradient-free ALDI SDE reads

\begin{equation}\label{eq:grad_free_aldi}
    \mathrm{d}u^{(j)}_t = -\tilde{\mathcal{A}}(U_t)\mathrm{d}t + \frac{d +1}{J}(u_t^{(j)}-\overline{u}_t)\mathrm{d}t + \sqrt{2}(\widetilde{C}_u)_t^{1/2}\mathrm{d}W_t^{(j)}\,.
\end{equation}
In the general nonlinear setting, the system retains affine invariance; however, there is no theoretical guarantee of convergence to the target distribution. Nonetheless, numerical experiments demonstrate good practical performance.

\begin{remark} 
The discussed interaction particle systems based on SDE formulations offer the advantage to define canonical proposals for Metropolis steps. We refer to~\cite{sprungk2025metropolis} for a related analysis. 
\end{remark}

\section{Ensemble Kalman inversion as optimizer for the EIG}\label{sec:Opt_EKI}

 We introduce Ensemble Kalman Inversion (EKI) in the continuous-time formulation proposed in \cite{Schillings2017}, and adapt it to our setting. EKI is a derivative-free optimization method inspired by the ensemble Kalman filter, and is particularly well-suited for inverse problems where gradient information is unavailable or expensive to compute. It approximates the minimizer of a regularized quadratic problem of the general form
 \begin{equation} \label{eqn:eqn_pot_regul_eki}
    \Psi^{\scaleto{\mathrm{reg}}{5pt}} (p):=\tfrac{1}{2}\|y-\mathcal{F}(p)\|^2_{\Gamma} + \tfrac{\alpha}{2}\|p\|_{C_p}^2\,,
\end{equation}
 by evolving an ensemble of particles using empirical covariances derived from the ensemble itself. 

To solve the EIG optimization problem \eqref{eq:general_optproblem} for $ \mathcal{L}_D(p)$  using EKI, we reformulate the maximization as a minimization problem, i.e. we minimize $ -\mathcal{L}_D(p) + c $, where $ c > \max_p \mathcal{L}_D(p)$ , and define the operator
\begin{equation}
\label{eq:EKI_loss}
\mathcal{F}:\mathcal{D} \to \mathbb{R}, \quad p \mapsto \sqrt{2(-\mathcal{L}_D(p) + c)}\,.
\end{equation}
We introduce a regularization of the design $p$ by 

\begin{equation}
\inf\limits_{p \in \mathcal{D}} \frac12\|\mathcal{F}(p)\|^2_2 + \tfrac{\alpha}{2} \|p\|^2_{C_p}\,,
\label{eq:regularised_optproblem}
\end{equation}
for a given regularization parameter $\alpha > 0 $ and symmetric positive-definite matrix $C_p \in \mathbb{R}^{\dim_p \times \dim_p}$. We assume in the following the existence of a unique minimizer $p^\ast \in \mathcal{D}$ of \eqref{eq:regularised_optproblem}.

Given an initial ensemble $p_0 = (p_0^{(i)})_{i \in \Jeki} \in \mathcal{D}^{\Nens}$ with number of ensemble particles $\Nens \in \mathbb{N}$ and index set $\Jeki := \{1, \ldots, \Nens\}$, EKI moves the particle according to the following dynamics

\begin{align} \label{EKI_regul_vi}
    \frac{\mathrm{d} p^{(i)}(t)}{\mathrm{d}t} &= (1-\rho)\left[-(\widetilde{C}_{p,\cF})_t \cF(p_t^{(i)})-(\widetilde{C}_p)_t \alpha C_p^{-1}p^{(i)}_t\right] \notag\\
    &+\rho\left[-(\widetilde{C}_{p,\cF})_t  \bar{\cF}(u_t)-(\widetilde{C}_p)_t \alpha C_p^{-1}\overline{p_t}\right] \qquad (i \in \Jeki)\\
    p(0) &= p_0\,, \notag
\end{align}
where $0\leq\rho< 1$. Here, variance inflation is considered to ensure the convergence under the following assumptions.

\begin{assumption} \label{assu:convex_lip}
    The regularized potential $\Psi^{\scaleto{\mathrm{reg}}{5pt}}(p)=\frac12\|\mathcal{F}(p)\|^2_2 + \tfrac{\alpha}{2} \|p\|^2_{C_p}$ is in $C^2(\mathcal{D},\mathbb{R}_+)$ satisfying
    \begin{enumerate}
        \item ($\mu$-strong convexity). There exists $\mu>0$ such that
            \[\Psi^{\scaleto{\mathrm{reg}}{5pt}}(p_1)-\Psi^{\scaleto{\mathrm{reg}}{5pt}}(p_2)\geq \langle \nabla \Psi^{\scaleto{\mathrm{reg}}{5pt}}(p_2),p_1-p_2\rangle+\frac{\mu}{2}\|p_1-p_2\|^2, \quad \forall\ p_1,p_2 \in \mathcal{D}\,.\]
        \item ($L$-smoothness). There exists $L>0$ such that $\nabla \Psi^{\scaleto{\mathrm{reg}}{5pt}}$ satisfies:
        \[\|\nabla \Psi^{\scaleto{\mathrm{reg}}{5pt}}(p_1)-\nabla \Psi^{\scaleto{\mathrm{reg}}{5pt}}(p_2)\|\le L \|p_1-p_2\|\quad \forall\ p_1,p_2 \in \mathcal{D}\,.\]
    \end{enumerate}
    Additionally we assume that the forward operator satisfies $\cF\in C^2(\mathcal{D},\mathbb{R}^\Nobs)$ is locally Lipschitz continuous and can linearly approximated, i.e.
    \begin{equation}\label{eqn:approx_error_eig_general}
    \cF(p_1)=\cF(p_2)+D\cF(p_2)(p_1-p_2)+\text{Res}(p_1,p_2) \quad \forall p_1,p_2\in \mathcal{D}\,,
\end{equation}
where $D\cF$ denotes the Fréchet derivative of $\cF$. The linear approximation error is bounded by
$   \|\text{Res}(p_1,p_2)\|_2\leq b_{\text{res}}\|p_1-p_2\|_2^2\,.
$
\end{assumption}

Then one obtains the following existence and uniqueness result as well as asymptotic convergence results for the EKI solutions.

\begin{theorem}[\protect{\cite{Weissmann_2022}}]\label{thm:conver_simon}
Define $V_e(t)=\frac{1}{\Jeki}\sum_{i=1}^\Jeki\frac{1}{2}\|e^{(i)}_t\|^2$, where $e^{(i)}_t=p^{(i)}_t-\overline{p}_t$ and assume that \cref{assu:convex_lip} holds. 
For an initial ensemble $\{p^{(1)}_0,...,p^{(\Nens)}_0\}$ the ODE system \eqref{EKI_regul_vi} has unique and global solutions $p^{(i)}(t)\in C^1([0,\infty);\mathcal{S})$ for all $i\in\Jeki$, where $\mathcal{S}=\bar p_0+\spann\{e^{(i)}_0, i\in\{1,\ldots,\Jeki\}\}$.\\
   Furthermore, there holds:
    \begin{enumerate}
        \item The ensemble collapses satisfies
        $V_e(t) \in \mathcal O(t^{-1})$\,.
        \item Denote by $p^*\in\mathcal{D}$ the unique minimizer of \eqref{eqn:eqn_pot_regul_eki} with respect to the subspace $\mathcal{S}$, then 
        \[\frac{1}{\Jeki}\sum_{i=1}^J\Psi^{\scaleto{\mathrm{reg}}{5pt}}\left(p^{(i)}_t\right)-\Psi^{\scaleto{\mathrm{reg}}{5pt}}\left(p^*\right)\in\mathcal O(t^{-\gamma})\,,\]
        where $0<\gamma<(1-\rho)\frac{L}{\mu}(\sigma_{\max}+c_{\mathrm{lip}}\lambda_{\max}\|\widetilde{C}\|_{HS})$. Here $\sigma_{\max},\lambda_{\max}$ denote the largest eigenvalues of $C_p^{-1}$ and $\Gamma^{-1}$, $\|\cdot\|_{HS}$ is the Hilbert-Schmidt norm.
    \end{enumerate}
\end{theorem}

\section{Full algorithm description} \label{sec:algo}

In this section, we illustrate the complete structure of our method.  
The algorithm consists of the following components.  
Given a sequence of observations $\bm{y}_n^\dagger =(y_1^\dagger, \ldots, y_{n}^\dagger)$ and associated designs $\bm{p}_n^\dagger=(p_1^\dagger,\ldots, p_n^\dagger)$, at each time step $n \in \mathbb{N}$, we proceed as follows:

\begin{enumerate}
    \item Given prior samples $\{u_n^{(j)}\}_{j=1}^J \sim \pi_n=\pi_{u\mid \bm{y}_n^\dagger, \bm{p}_n^\dagger}$, compute joint samples $\{(u_n^{(j)}, y_n^{(i,j)})\}_{j=1}^J$ by evaluating the forward model with noise:
    \[
    y_n^{(i,j)} = G_n(u_n^{(j)}, p_{n,0}^{(i)}) + \eta_n^{(i,j)}
    \]
    for each initial design $p_{n,0}^{(i)}$.

    \item Compute the Gaussian approximation from~\eqref{eqn:gaussapprox_joint} to the joint samples \\$\{(u_n^{(j)}, y_n^{(i,j)})\}_{j=1}^J$ to approximate both the marginal and posterior distributions.

    \item Compute the difference between the upper and lower bounds of the EIG. If the gap is too large, consider using a parametrized Laplace approximation to improve the lower bound. 
    
    \item If the discrepancy stays above a given tolerance, explore alternative approximations of the joint distribution  (e.g., Gaussian mixtures) to improve the accuracy of the marginal and posterior approximations and thus of the upper and lower bounds.

    \item Update the designs $\{p_{n,\Teki}^{(i)}\}_{i=1}^\Nens$ via EKI as described in~\eqref{EKI_regul_vi} and \cref{algo:eki_reg} based on the estimate $\mathrm{EIG}(p_n^{(i)})$ from \eqref{eqn:eig_approx} as detailed in \cref{algo:eig_estimator}.

    \item Select the optimal design $p_{n+1}^\dagger$ from the final ensemble $\{p_{n,\Teki}^{(i)}\}_{i=1}^\Nens$, for example by computing the ensemble mean or choosing the maximizer:
   $$ 
    p_{n+1}^\dagger = \arg\max_i \widetilde{\mathrm{EIG}}(p_{n,\Teki}^{(i)}).
    $$
    \item Obtain measurement $y_{n+1}^\dagger$ in the sequential process for the design choice $p_{n+1}^\dagger$.
    \item Given $y_{n+1}^\dagger$ and $p_{n+1}^\dagger$, update $\bm{y}_{n+1}^\dagger$ and $\bm{p}_{n+1}^\dagger$ compute posterior samples $\{u_{n+1}^{(j)}\}_{j=1,\ldots,J}$ using the ALDI sampler, as described in \cref{algo:aldi}. The samples, correspond to prior samples of $\pi_{n+1}$.
\end{enumerate}

The complete algorithm is summarized in \eqref{algo:end}.
\begin{algorithm}[h]
  \caption{EIG estimation via variational approximations}\label{algo:eig_estimator}
    \begin{algorithmic}[1]
  \Require $\left\{\begin{array}{ll}
            n, & \text{number of data assimilation window}, \\
           G_n, & \text{forward operator}, \\
           \mathcal{N}(0,\Gamma), & \text{data noise distribution}, \\
            p & \text{design parameter}, \\
            (u_n^{(j)})_{j=1}^J\sim\pi_{n}, & \text{samples from sequentially updated prior}.\\
            \delta & \text{tolerance for upper and lower bound}\\
            \pi_0, & \text{initial prior density}, \\
           \bm{y}_{n-1}^\dagger,\bm{p}_{n-1}^\dagger, & \text{previous observations and designs},
            \end{array}
    \right. $
  \Ensure Estimate $ \widetilde{\mathrm{EIG}}_{n}(p) \approx \mathrm{EIG}(p, \mu_n)$.
        \State Simulate data: $y_{n}^{(j)} = G_n(u_{n}^{(j)}, p) + \eta_{n}^{(j)}$ for $j=1,\ldots,J$, $\eta_{n}^{(j)}\sim \mathcal{N}(0,\Gamma)$.
        \State Fit Gaussian approximation $\tilde{\mu}_{u_n,y_n \mid p}^{(J)}=\mathcal{N}(\widetilde m_{u_n,y_n\mid p}^{(J)},\widetilde C_{u_n,y_n\mid p}^{(J)})$ from $\{(u_n^{(j)},y_n^{(j)})\}_{j=1}^J$.
        \State From $\tilde{\mu}_{u_n,y_n \mid p}^{(J)}$ compute marginal approximation $\tilde{\mu}_{y_n \mid p}^{(J)}=\mathcal{N}(\widetilde m_{y_n\mid p}^{(J)},\widetilde C_{y_n\mid p}^{(J)})$ and posterior approximation $\tilde{\mu}_{u_n \mid y_n, p}^{(J)}=\mathcal{N}(\widetilde m_{u_n \mid y_n, p}^{(J)},\widetilde C_{u_n \mid y_n, p}^{(J)})$.
        \For{For each $j=1,\ldots,J$}:
            \State Evaluate likelihood: $\pi_{y_n \mid u_n, p}(y_n^{(j)}) = \mathcal{N}(y_n^{(j)}; G(u_n^{(j)}, p), \Gamma)$.
            \State Approx. marginal density $\widetilde{\pi}_{y_n \mid p}(y_n^{(j)}) = \mathcal{N}(y_n^{(j)}; \widetilde m_{y_n\mid p}^{(J)}, \widetilde C_{y_n\mid p}^{(J)}).$
            \State Approx. posterior density $\widetilde{\pi}_{u_n \mid y_n, p}(u_n^{(j)}) = \mathcal{N}(u_n^{(j)}; \widetilde m_{u_n \mid y_n, p}^{(J)}, \widetilde C_{u_n \mid y_n, p}^{(J)}).$
            \State Evaluate sequential prior $\pi_{n-1}(u_n^{(j)})$ (reuse function evaluations from ALDI).
        \EndFor
    \State Compute upper and lower bound
            {\scriptsize\[
            \widetilde{\mathrm{EIG}}_n^{\text{UB}}(p)= \frac{1}{J} \sum_{j=1}^J \log\left(\frac{\pi_{y_n \mid u_n^{(j)}, p}(y_n^{(j)})}{\widetilde{\pi}_{y_n \mid p}(y_n^{(j)})} \right),\widetilde{\mathrm{EIG}}_{n}^{\text{LB}}(p)= \frac{1}{J} \sum_{j=1}^J \log\left( \frac{\widetilde{\pi}_{u_n \mid y_n, p}(u_n^{(j)})}{\pi_{n-1}(u_n^{(j)})} \right)
            \]
            }\If{$|\widetilde{\mathrm{EIG}}_{n}^{\text{UB}}(p)-\widetilde{\mathrm{EIG}}_{n}^{\text{LB}}(p)|>\delta$}
            
            \State Use parametrized Laplace approximation to improve lower bound $\widetilde{\mathrm{EIG}}_{n}^{\text{LB, Laplace}}(p).$
                        \If{$|\widetilde{\mathrm{EIG}}_{n}^{\text{UB}}(p)-\widetilde{\mathrm{EIG}}_{n}^{\text{LB, Laplace}}(p)|>\delta$}
            
            \State Improve approximation for the joint distribution (e.g. Gaussian mixtures).
            \EndIf
            \EndIf
        
        \State Set $\widetilde{\mathrm{EIG}}_{n}(p) = \widetilde{\mathrm{EIG}}_{n}^{\text{UB}}(p)$.
  \end{algorithmic}
\end{algorithm}

\begin{algorithm}[h]
  \caption{Sequential gradient-free BOED with interacting particle systems}\label{algo:end}
    \begin{algorithmic}[1]
  \Require $\left\{\begin{array}{ll}
            N, & \text{length of data assimilation window}, \\
            \widetilde{\mathrm{EIG}}& \text{BOED loss estimators of \cref{algo:eig_estimator}, } \\
            c_n, & \text{upper bounds for $\widetilde{\mathrm{EIG}}$ realizations from \eqref{eq:EKI_loss}},\\
            \{p^{(i)}_0\}_{i=1}^\Nens, & \text{initial design ensemble of ensemble size $\Nens$}.\\
            \Teki & \text{time horizon for EKI}.\\
            \bm{y}_n^\dagger, & \text{observed data}.
            \end{array}
    \right. $
  \Ensure Optimized design proposals $\{p_n^\dagger\}_{n=1}^N$
  \State Generate samples $\{u_{0}^{(j)}\}_{j=1}^J\sim \pi_{0}$
  \For{$n$ in $\{0,\ldots, N-1\}$}
  \gray{\Comment{sequential BOED loop}}   
        \State Define $\cF_{n}(\cdot) = \sqrt{2(-\widetilde{\mathrm{EIG}}(\cdot,\pi_{n}) + c_n)}$ from \eqref{eq:EKI_loss}.
        \State Compute $\{p^{(i)}_{n,\Teki}\}_{i=1}^\Nens = \mathrm{ReguarlizedEKI}(\mathcal{F}_{n}, \{p_{n,0}^{(i)}\}_{i=1}^\Nens, \{u_{n}^{(j)}\}_{j=1}^J)$ using \cref{algo:eki_reg}.
        \State Choose optimal design $p_{n+1}^\dagger = \arg\max_i \widetilde{\mathrm{EIG}}(p_{n,\Teki}^{(i)})$.
        \State Define $\pi_{n+1} = \pi_{u\mid y_{n}^\dagger,p_{n}^\dagger}
        \propto \exp(-\tfrac{1}{2}\|y_{n}^\dagger - G(\,\cdot\,, p_{n}^\dagger)\|_{\Gamma}^2)\pi_{n}$ as in \eqref{eqn:post_dist}.
        \State Generate samples $\{u_{n+1}^{(j)}\}_{j=1}^J\sim \pi_{n+1}$ using $\mathrm{ALDI}(J)$ (\cref{algo:aldi}).
  \EndFor
  \end{algorithmic}
\end{algorithm}

\section{Numerical experiments} \label{sec:Numerics}

To demonstrate the practical performance of our \cref{algo:end} and to support the theoretical results, we present numerical experiments in three settings of increasing complexity: a linear Gaussian model, a near-linear Gaussian model, and a one-dimensional heat equation. These examples highlight both the robustness and limitations of our approach under varying degrees of nonlinearity and model structure.

\subsection{Linear Gaussian Case}\label{sec:linear_case}

We consider a one-dimensional linear Gaussian model of the form
\[
G(u, p) = A(p) u\,,
\]
where the parameter-dependent linear operator is defined as
\[
A(p) = -c(p - 1)^2 + d\,,
\]
for constants \( c > 0 \) and \( d \in \mathbb{R} \).

This example allows us to illustrate the behavior of the upper and lower bounds on the expected information gain (EIG), as introduced in \cref{lem:eig_upper_lower_bound_linear} and given explicitly in \eqref{eqn:eig_upper_bound} and \eqref{eqn:eig_lower_bound}. These bounds converge to the exact EIG as the Gaussian approximations of the marginal and posterior distributions become more accurate.

To demonstrate this convergence, we plot the Kullback–Leibler (KL) divergences:
\[
D_{\text{KL}}\left(\mu_{y \mid p} \,\|\, \tilde{\mu}_{y \mid p}^{(J)}\right) \quad \text{and} \quad
D_{\text{KL}}\left(\mu_{u \mid y, p} \,\|\, \tilde{\mu}_{u \mid y, p}^{(J)}\right)
\]
as functions of the number of samples \( J \).

We use a Gaussian prior with mean \( m_0 = 2 \), covariance \( \Sigma_0 = 2 \), and observation noise covariance \( \Gamma = 1 \). The number of joint samples \( J \) ranges from $\num{1e1}$ to \( \num{1e5} \).

\begin{figure}[htbp]
    \centering
    \includegraphics[width=0.8\textwidth]{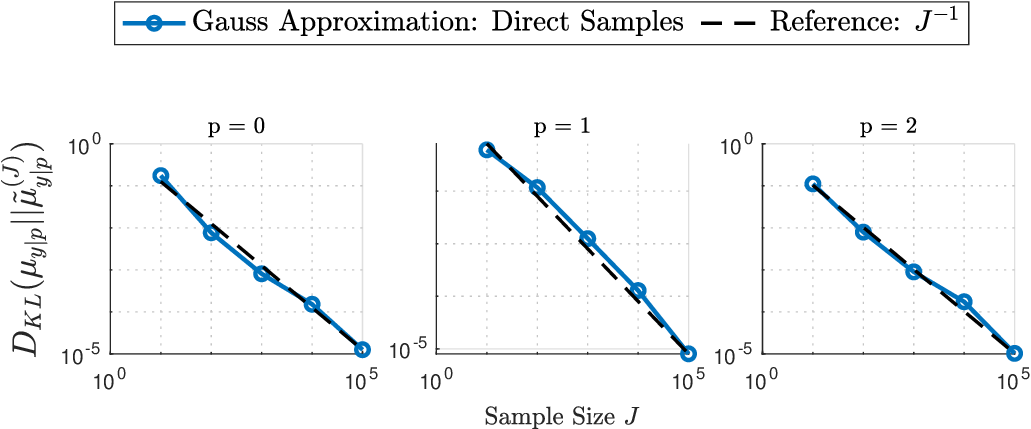}
    \caption{Linear case: KL divergence between the true marginal and the Gaussian approximation at $p \in \{0, 1, 2\}$ as a function of the number of particles \( J \). Results are averaged over 10 runs and shown in log-log scale.}
    \label{fig:GaussApprox_lin_marginal}
\end{figure}

\begin{figure}[htbp]
    \centering
    \includegraphics[width=0.8\textwidth]{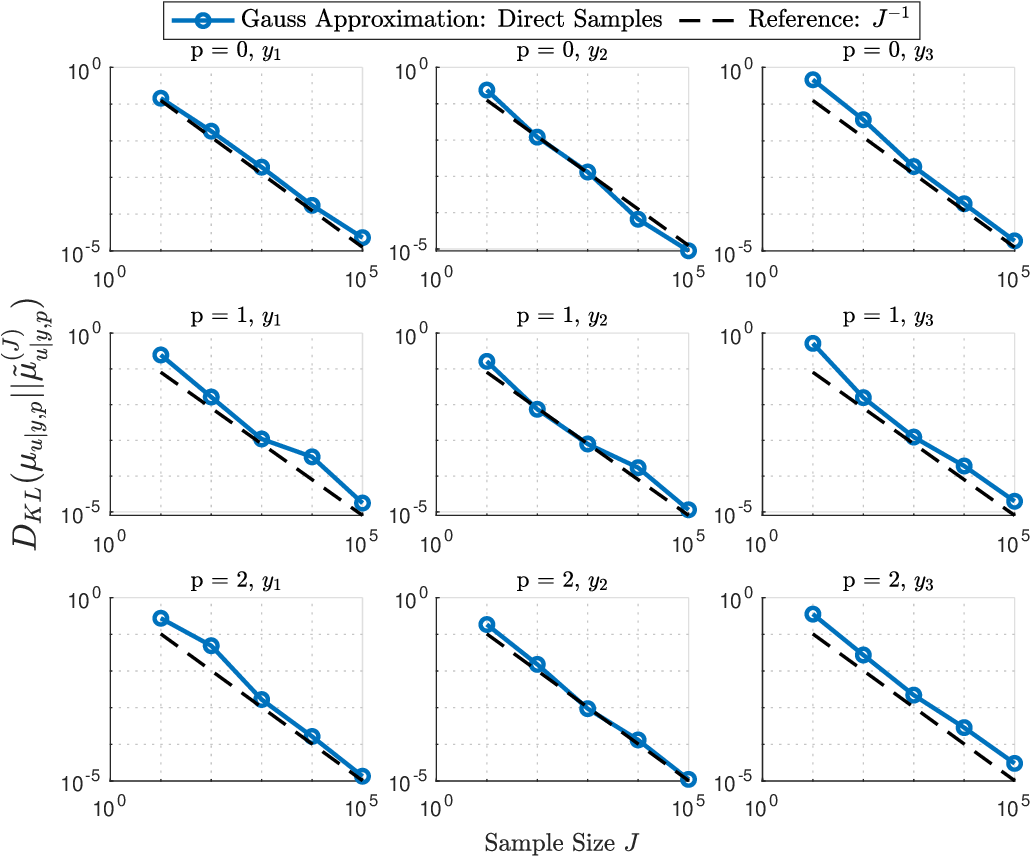}
    \caption{Linear case: KL divergence between the true posterior and its Gaussian approximation at $p \in \{0, 1, 2\}$ for data corresponding to the 10\textsuperscript{th}, 50\textsuperscript{th}, and 90\textsuperscript{th} percentiles of simulated observations. Results are shown in log-log scale, averaged over 10 runs.}
    \label{fig:GaussApprox_posterior}
\end{figure}

\Cref{fig:GaussApprox_lin_marginal} and \cref{fig:GaussApprox_posterior} show the convergence behavior of the Gaussian approximations to the marginal and posterior distributions. We observe convergence at rate \( \mathcal{O}(J^{-1/2}) \), as predicted in \cref{lem:eig_upper_lower_bound_linear}.

We further evaluate the performance of EKI for optimizing the EIG in this linear setting. We use $J = \num{1e5}$ joint samples and set the regularization parameter \( \alpha = \num{1e-2} \), design covariance \( C_p = 1 \), and number of EKI particles \( \Jeki = 3 \). The initial ensemble is randomly sampled from the interval $[0,2]$, and EKI is integrated using MATLAB's \texttt{ode45} solver. We use a time-dependent variance inflation \( v(t) = 0.01 \big(1 - (\tfrac{t}{\Teki} + 1)^{-\gamma} \big) \), with \( \gamma = 0.2 \) and \( \Teki = \num{1e5} \).

\begin{figure}[htbp]
    \centering
    \begin{subfigure}{0.45\textwidth}
        \centering
        \includegraphics[width=\linewidth]{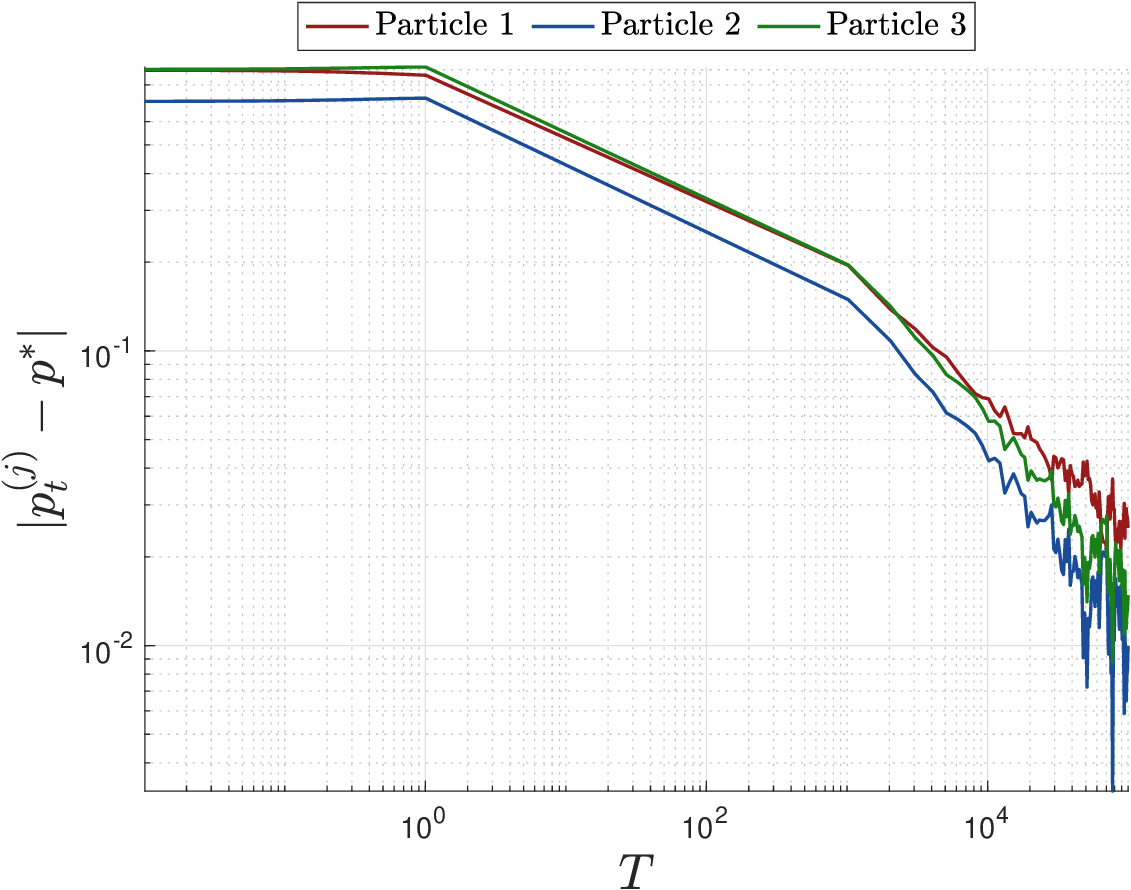}
    \end{subfigure}
    \hfill
    \begin{subfigure}{0.45\textwidth}
        \centering
        \includegraphics[width=\linewidth]{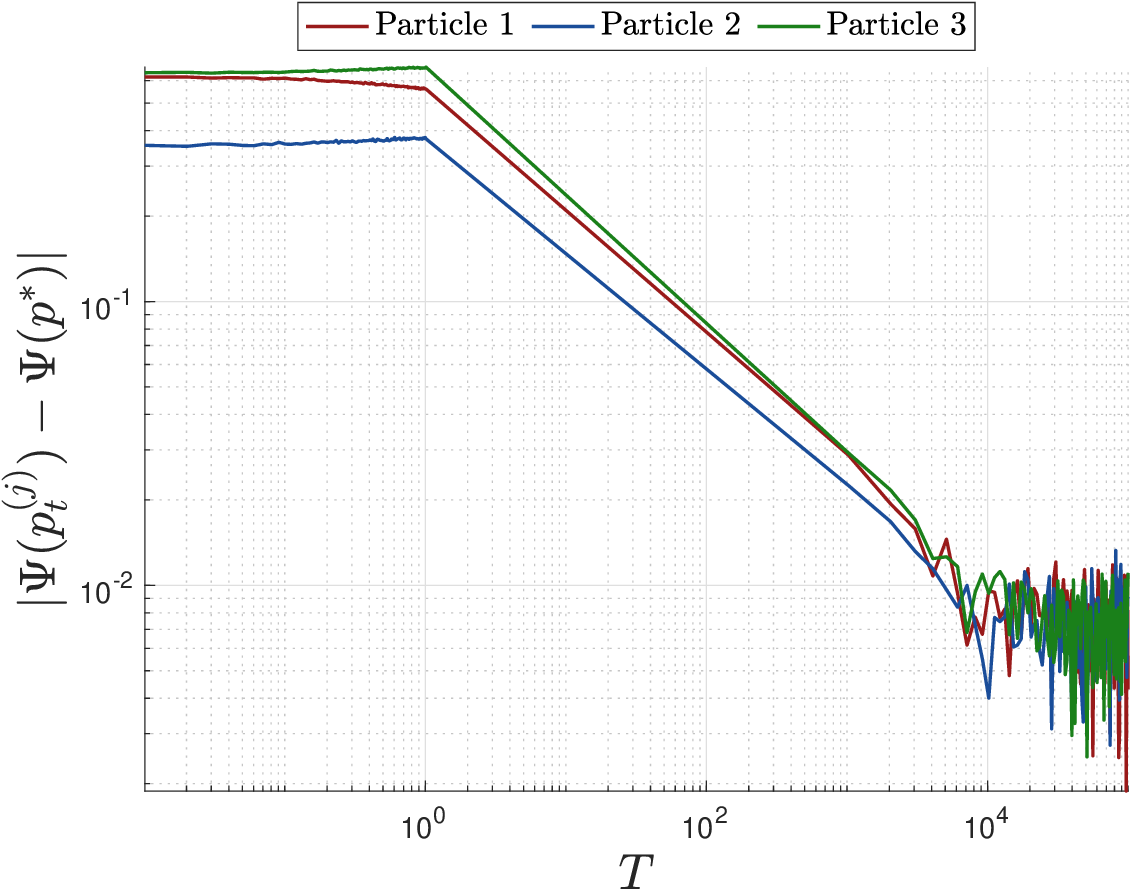}
    \end{subfigure}
    \caption{Linear case: EKI error in design space (left) and observation space (right).}
    \label{fig:eki_direct_linear_v2}
\end{figure}

\begin{figure}[htbp]
    \centering
    \includegraphics[width=0.6\textwidth]{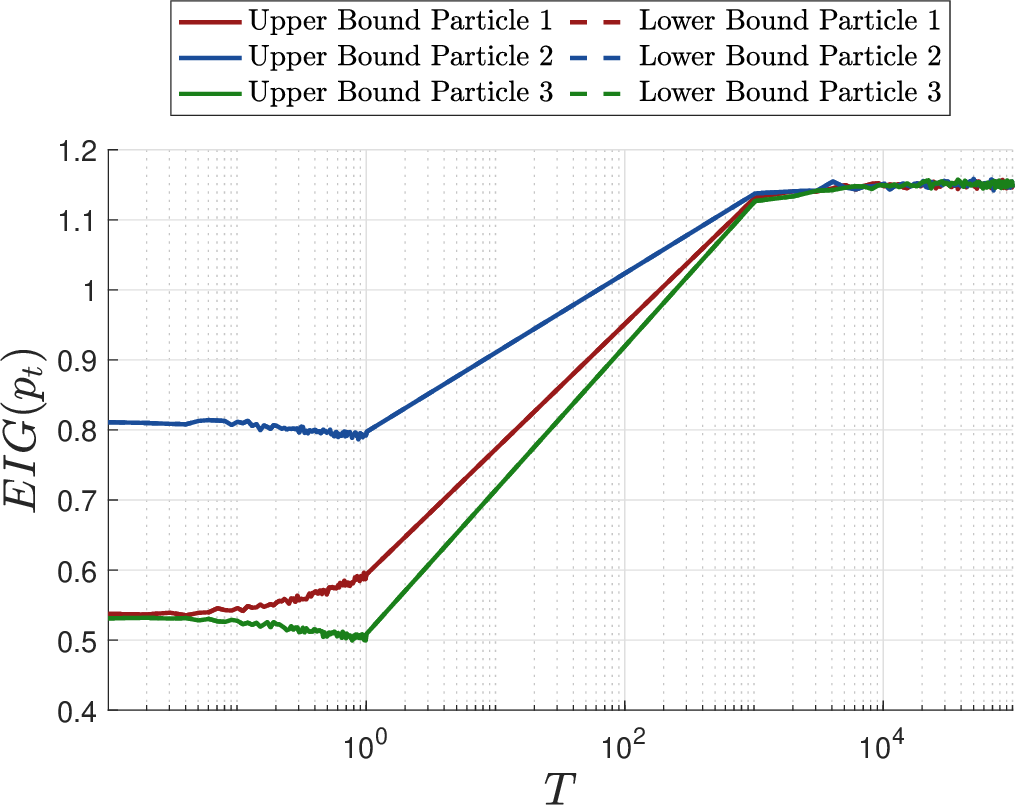}
    \caption{Linear case: Upper and lower bounds on the EIG for the evolving EKI solution \( p_t^{(i)} \), shown in semi-log-x scale.}
    \label{fig:eki_linear_ub_lb_v2}
\end{figure}

\Cref{fig:eki_direct_linear_v2} displays the convergence of EKI in both design and observational spaces. The true optimum \( p^* \) is computed using MATLAB’s \texttt{fminunc} and a high-fidelity double-loop Monte Carlo approximation with \( J_{\text{ref}} = \num{1e7} \) samples. The design error is given by \( |p_t^{(i)} - p^*| \), and the observation error by \( |\Psi^{\mathrm{reg}}(p_t^{(i)}) - \Psi^{\mathrm{reg}}(p^*)| \), with \( \Psi^{\mathrm{reg}} \) defined in \eqref{eqn:eqn_pot_regul_eki}.

Finally, \cref{fig:eki_linear_ub_lb_v2} shows how the upper and lower EIG bounds become tighter as the EKI solution converges. This confirms the consistency of the variational bounds discussed in \cref{lem:eig_upper_lower_bound_linear}.

\subsection{Near-Linear Example}\label{sec:near_linear}

We now extend the linear case by introducing a small nonlinear term to the forward model:
\[
G(u, p) = A(p)u + \tau u^2\,,
\]
where the linear component is given by
\[
A(p) = -c(p - 1)^2 + d\,.
\]

The goal of this experiment is to investigate how the upper and lower bounds of the EIG, given in \eqref{eqn:eig_upper_bound} and \eqref{eqn:eig_lower_bound}, behave as the nonlinearity vanishes. In particular, we study the effect of decreasing \( \tau \) and expect the bounds to converge to those of the linear model as \( \tau \to 0 \). Additionally, increasing the number of samples \( J \) should lead to tighter bounds due to improved approximations of the marginal and posterior distributions.

In our numerical experiments, we set \( c = 2 \), \( d = c + 1 \), and use a Gaussian prior with mean \( m_0 = 2 \), covariance \( \Sigma_0 = 1 \), and observational noise covariance \( \Gamma = 1 \).

\begin{figure}[htbp]
    \centering
    \includegraphics[width=0.9\textwidth]{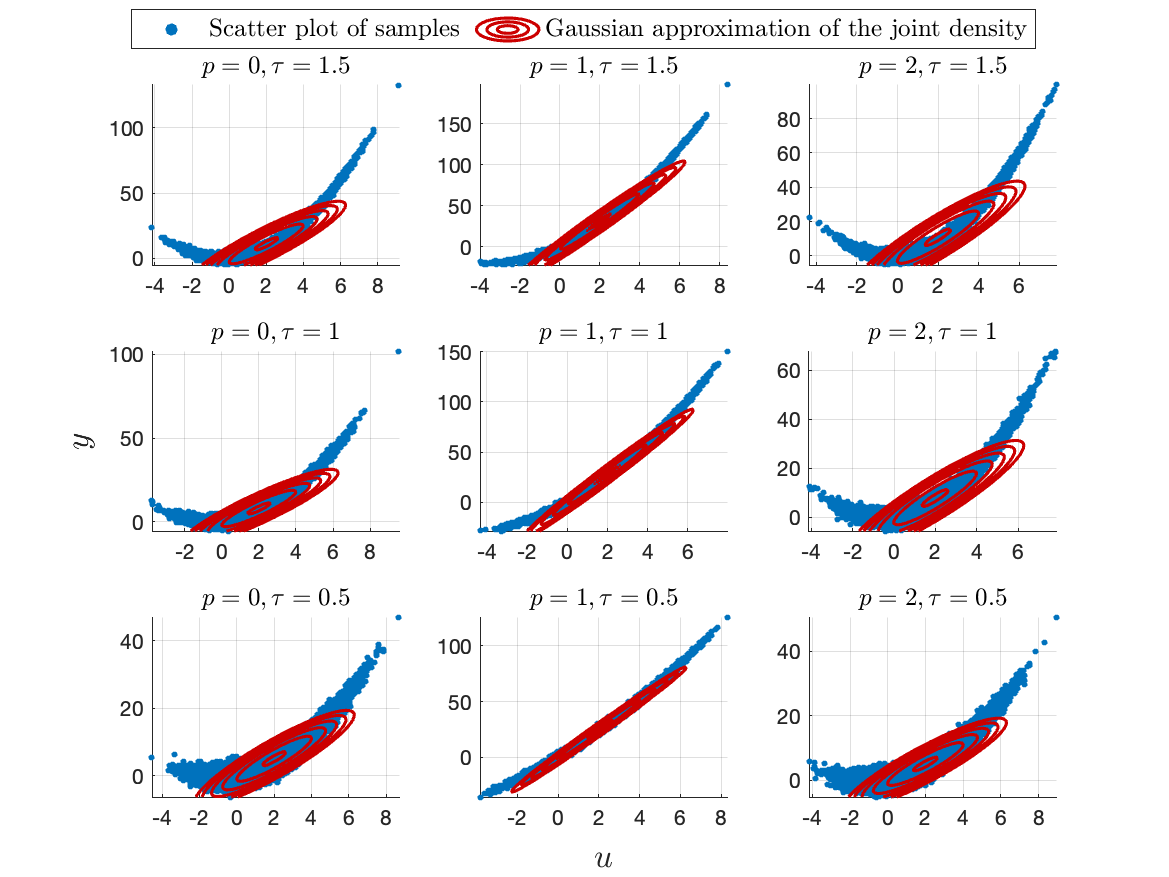}
    \caption{Near-linear example: Joint densities for $p = \{0, 1, 2\}$ and varying nonlinearity levels $\tau = \{1.5, 1, 0.5\}$.}
    \label{fig:eig_near_linear_joint}
\end{figure}

\begin{figure}[htbp]
    \centering
    \includegraphics[width=0.9\textwidth]{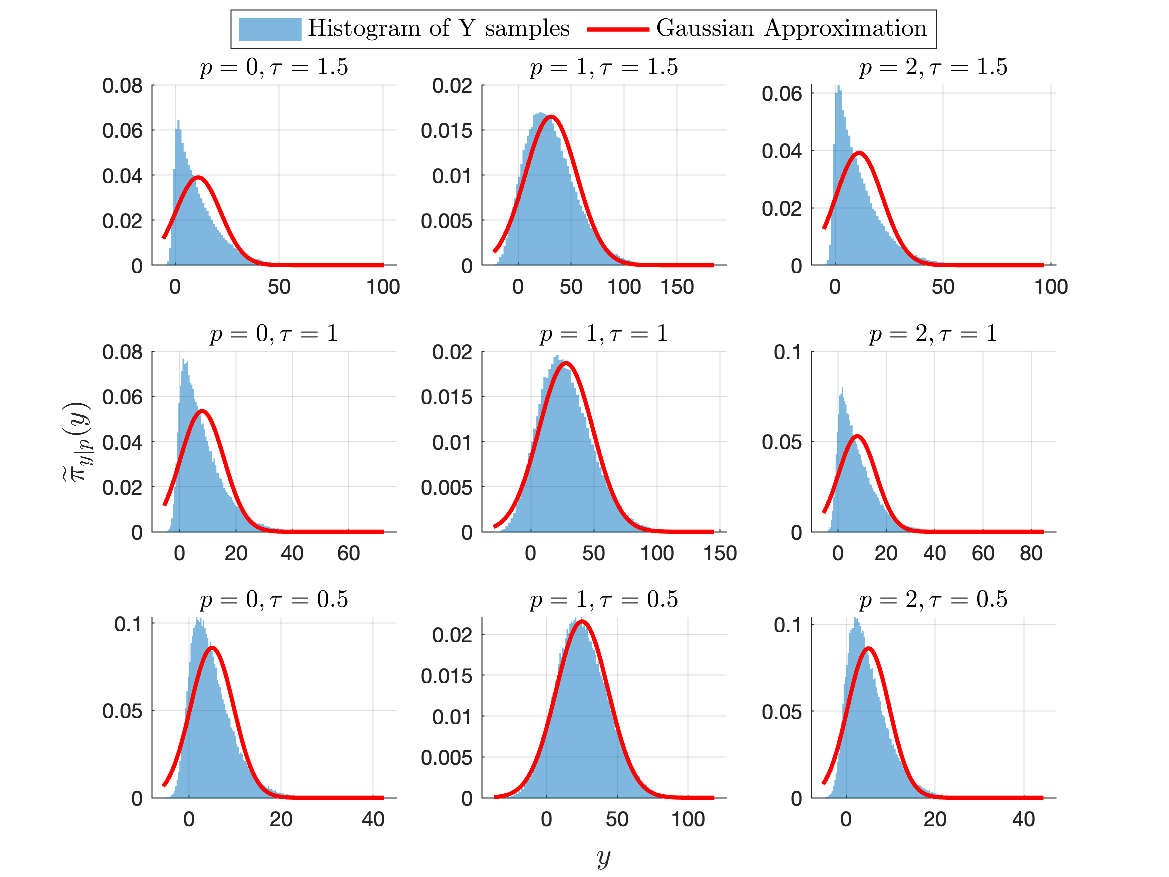}
    \caption{Near-linear example: Marginal densities for $p = \{0, 1, 2\}$ and $\tau = \{1.5, 1, 0.5\}$.}
    \label{fig:eig_near_linear_marginal}
\end{figure}

\begin{figure}[htbp]
    \centering
    \includegraphics[width=0.9\textwidth]{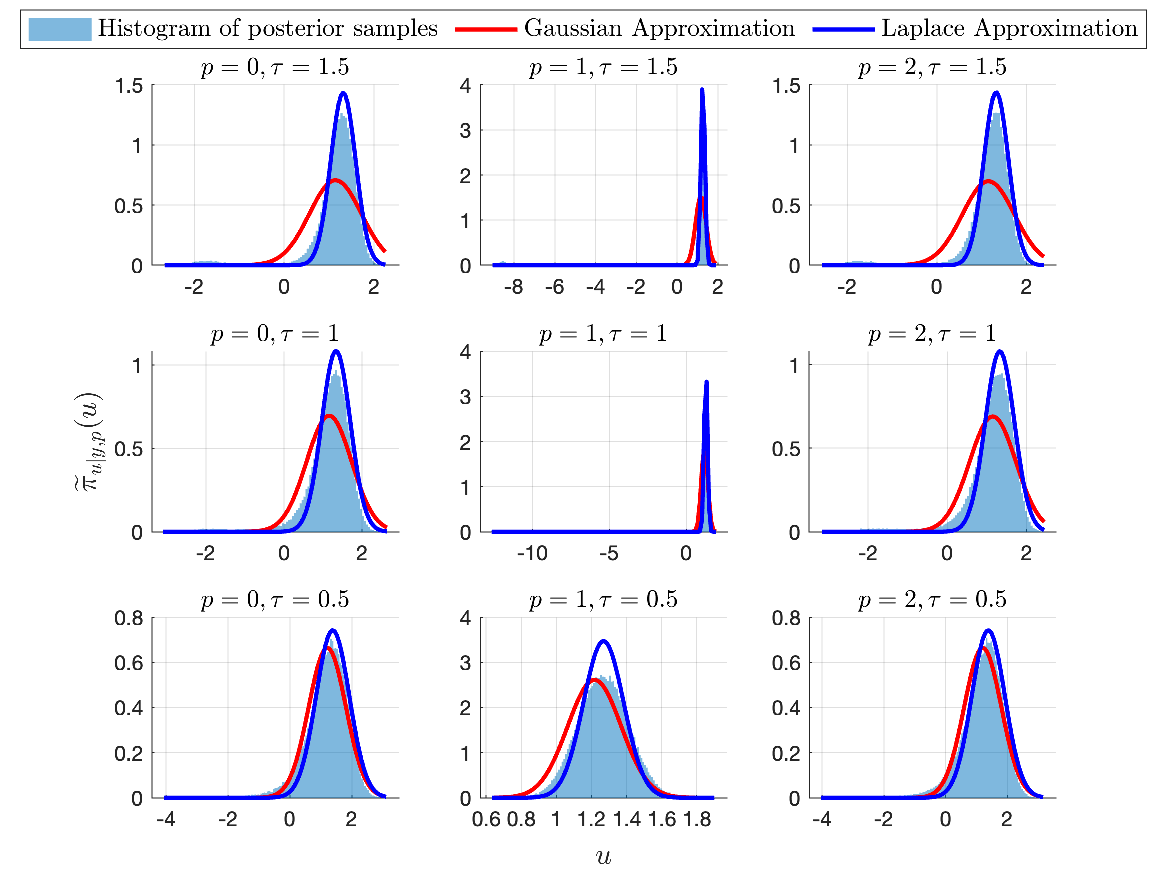}
    \caption{Near-linear example: Posterior densities for $p = \{0, 1, 2\}$ and $\tau = \{1.5, 1, 0.5\}$. Conditional data $y$ is chosen as the 30th percentile of the ordered joint samples.}
    \label{fig:eig_near_linear_posterior}
\end{figure}

\Cref{fig:eig_near_linear_joint,fig:eig_near_linear_marginal,fig:eig_near_linear_posterior} illustrate the effect of varying the nonlinearity parameter \( \tau \) on the joint, marginal, and posterior distributions. As expected, the Gaussian approximation improves with decreasing \( \tau \), as assumed in the near-linear setting of \eqref{ass:linear_apprx_forward_operator}.

For the posterior in particular, the Laplace approximation provides a more accurate fit than the Gaussian approximation, motivating a hybrid approach: start with the computationally efficient Gaussian approximation and switch to a Laplace approximation when higher accuracy is required.

\begin{figure}[htbp]
    \centering
    \includegraphics[width=0.9\textwidth]{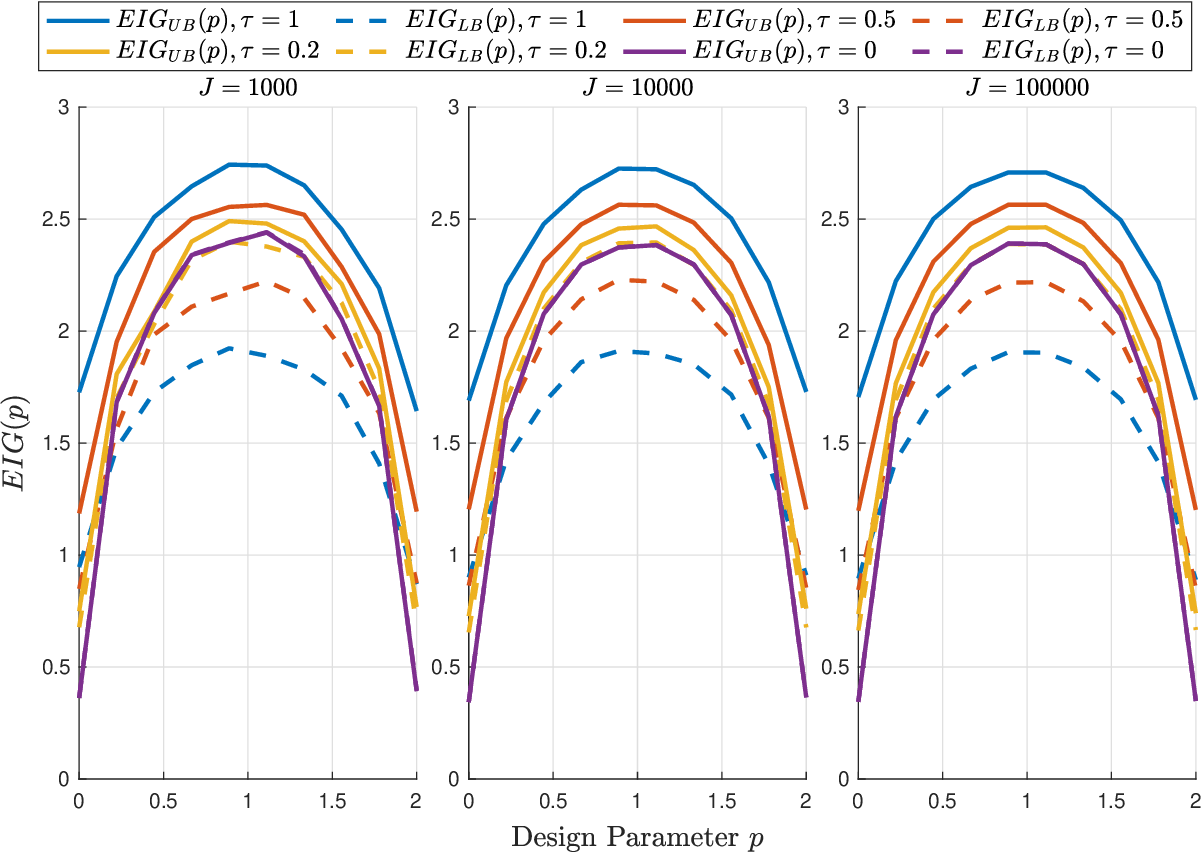}
    \caption{Near-linear example: EIG bounds computed via Gaussian approximation for different values of $\tau$.}
    \label{fig:eig_near_linear_Gauss_Approx}
\end{figure}

\begin{figure}[htbp]
    \centering
    \includegraphics[width=0.9\textwidth]{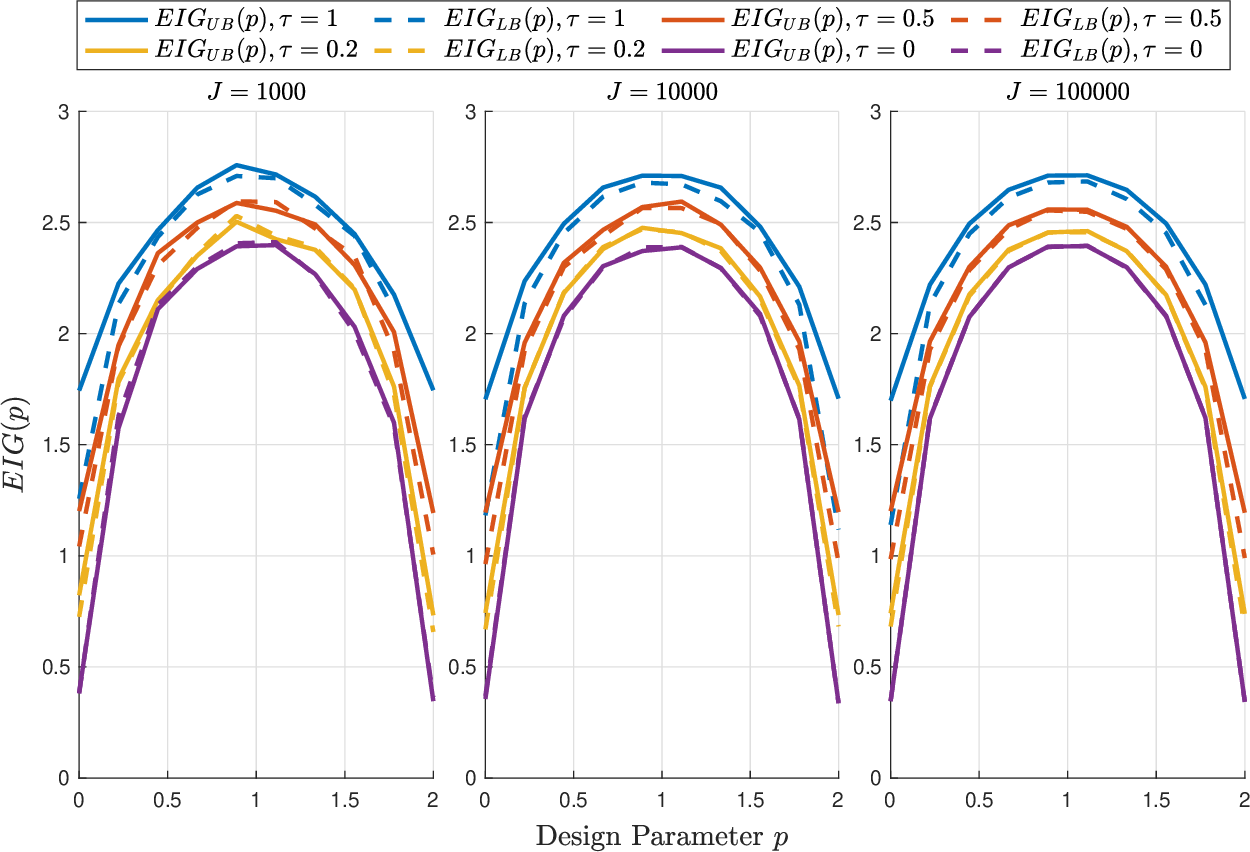}
    \caption{Near-linear example: EIG bounds using Laplace approximation for the posterior (lower bound) and Gaussian approximation for the marginal (upper bound).}
    \label{fig:eig_near_linear_Laplace_Approx}
\end{figure}

\Cref{fig:eig_near_linear_Gauss_Approx} shows the EIG bounds under Gaussian approximation for various levels of nonlinearity. As \( \tau \to 0 \), the bounds approach the linear case. \Cref{fig:eig_near_linear_Laplace_Approx} highlights that the Laplace approximation significantly improves the lower bound across all values of \( \tau \), consistent with the posterior density comparisons in \cref{fig:eig_near_linear_posterior}.

However, the Laplace approximation requires solving an optimization problem for each observation, making it computationally more expensive than the Gaussian alternative.

We also analyze EKI-based design optimization in this near-linear setting. We use \( J = \num{1e5} \), \( \Teki = \num{1e5} \), and set \( \tau = 1 \). The design and observation space errors are shown in the following figures.

\begin{figure}[htbp]
    \centering
    \begin{subfigure}{0.45\textwidth}
        \centering
        \includegraphics[width=\linewidth]{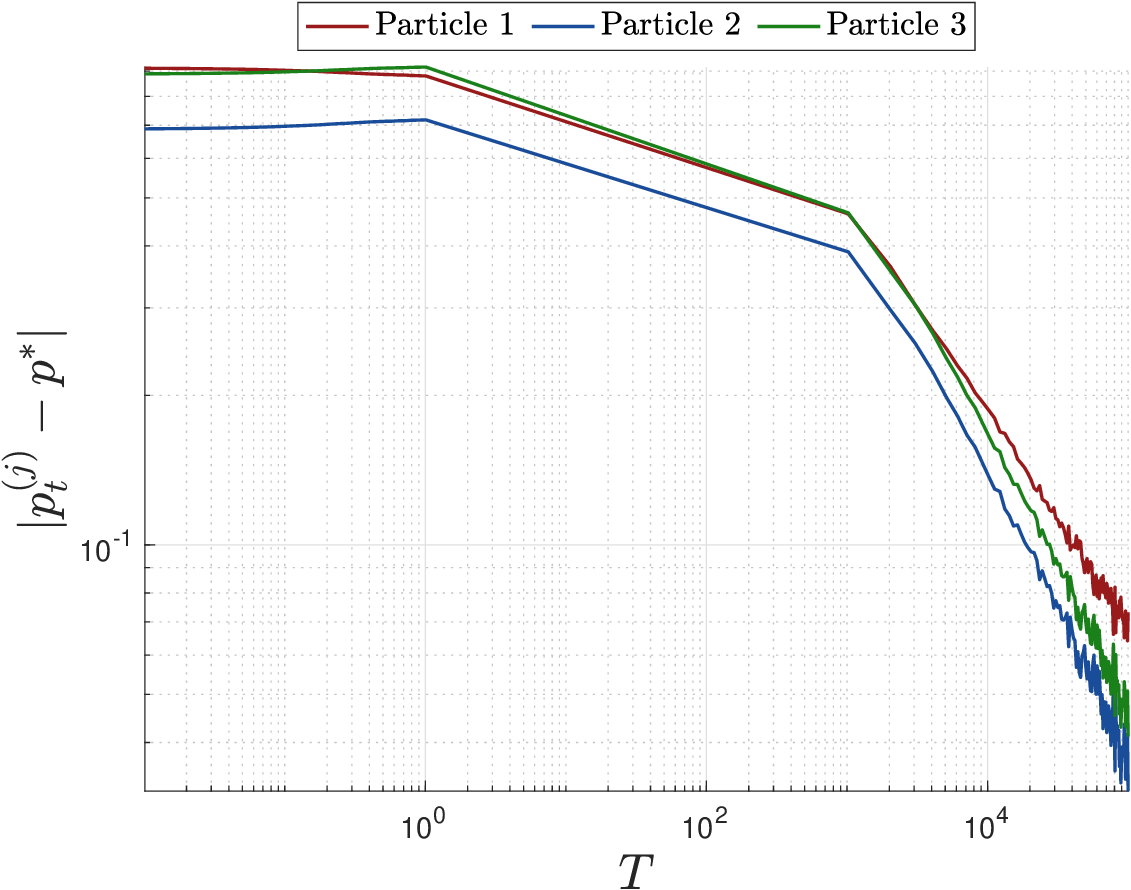}
    \end{subfigure}
    \begin{subfigure}{0.45\textwidth}
        \centering
        \includegraphics[width=\linewidth]{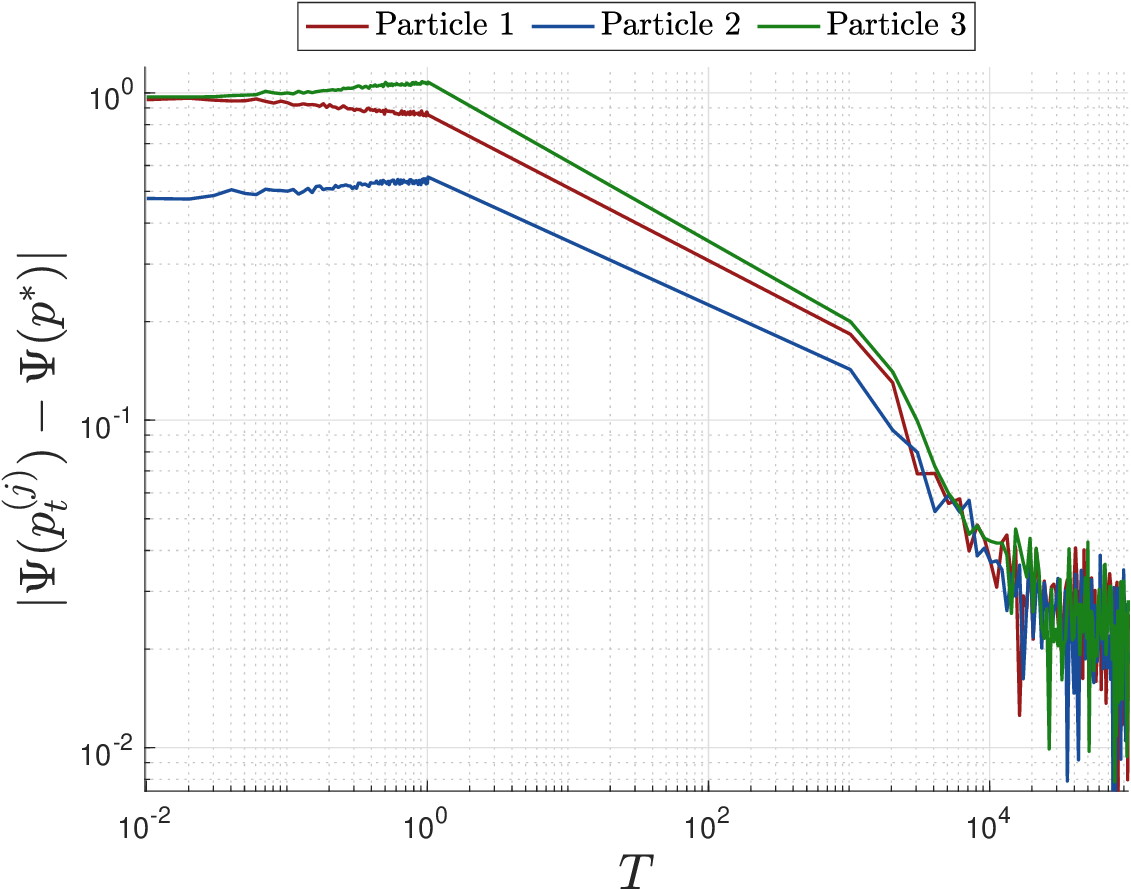}
    \end{subfigure}
    \caption{Near-linear example: EKI error in design space (left) and observation space (right).}
    \label{fig:eki_direct_near_linear_v2}
\end{figure}

\begin{figure}[htbp]
    \centering
    \begin{subfigure}{0.45\textwidth}
        \centering
        \includegraphics[width=\linewidth]{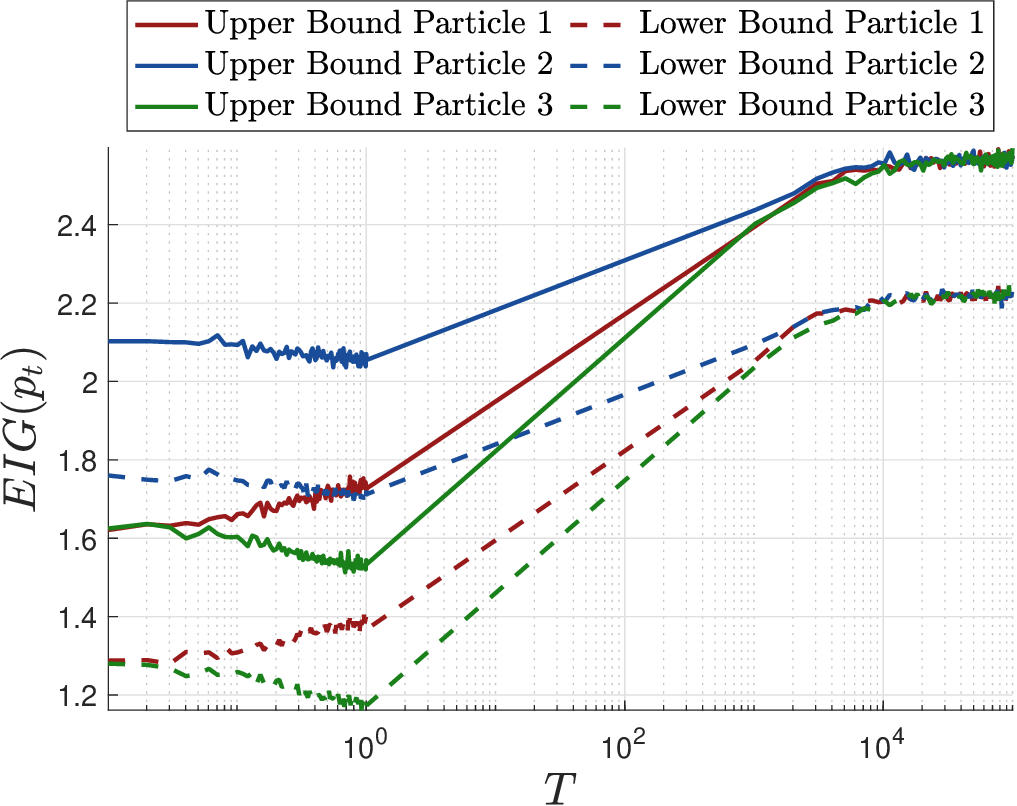}
        \caption{Lower bound via Gaussian approximation}
    \end{subfigure}
    \begin{subfigure}{0.45\textwidth}
        \centering
        \includegraphics[width=\linewidth]{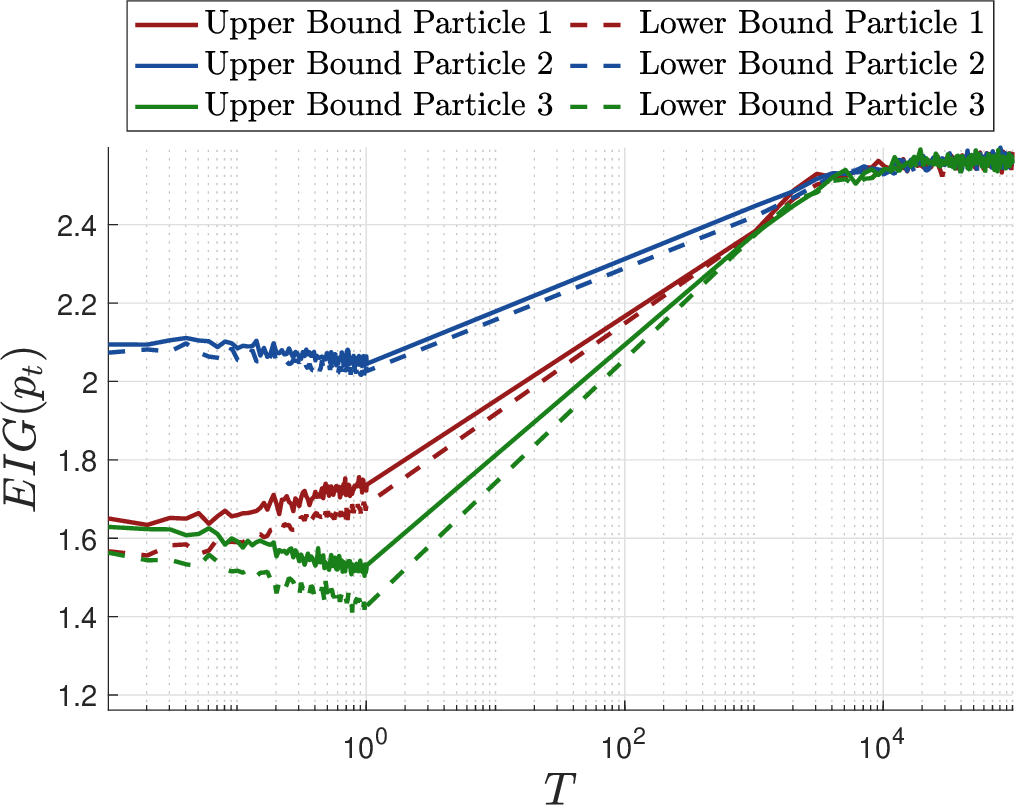}
        \caption{Lower bound via Laplace approximation}
    \end{subfigure}
    \caption{Near-linear example: Upper and lower bounds of EIG for the EKI design trajectory.}
    \label{fig:eki_near_linear_ub_lb_ga_la}
\end{figure}

\cref{fig:eki_near_linear_ub_lb_ga_la} demonstrates how increased nonlinearity (e.g., \( \tau = 0.5 \)) leads to larger discrepancies between bounds. The Laplace-based lower bound is consistently tighter, further supporting its use in near-linear but non-Gaussian settings.

\subsection{1D Heat Equation}
In this section, we consider a one-dimensional heat equation described by the following PDE:
\begin{alignat*}{2}
    \frac{\partial y(t,x)}{\partial t} - \nabla \cdot\left(\kappa(t,x,u)\nabla y(t,x)\right) &= f(t,x,p) \qquad && (t > 0,\, x \in (0,1)) \\
    y(0,x) &= 0 \qquad && (x \in (0,1)) \\
    y(t,0) = u(t,1) &= 0 \qquad && (t \geq 0)\,,
\end{alignat*}
where \( \kappa(t, x, u) \) denotes the diffusivity, parameterized by the unknown coefficient \( u \in \mathcal{X} \), and \( f(t, x, p) \) is a source term depending on the design parameter \( p \in \mathcal{D} \). \\

We consider
\begin{align*}
    \kappa(t,x,u) &= \exp(\sum_{\ell=1}^da_\ell(x)u_i) \\
    a_\ell(x)&=c \tfrac{1}{\ell^2} \cos(\ell\pi x)\\
    f(t,x,p) &= \exp\left(-\alpha(p-1)^2\right) \rho_{\sigma}(x)
\end{align*}

with $\alpha>0$ and $\rho_{\sigma}(x)$ denoting the Gaussian density of  $\mathcal{N}(0.5,\sigma)$ for $\sigma>0$.\\

We fix a final simulation time  $T_{\mathrm{end}}$  and discretize the time domain using a uniform step size  $\Delta t = 0.005$ . Instead of continuously observing the solution, we define a finite set of sequential observation times  $\{ t_1, t_2, \ldots, t_N \} \subset [0, T_{\mathrm{end}}]$ , at which measurements are taken. To solve the PDE, we employ a midpoint time discretization scheme. At each time step, the spatial domain is discretized using the finite element method (FEM), and the solution \( u_{i+1} \) is computed iteratively. 

At each observation step  $n$ , we aim to select an informative design  $p_n^\dagger$  by maximizing the EIG, computed under the current prior belief. To do so, we evaluate the EIG via EKI using $\Jeki$ particles. Given the EKI solutions $\{p_n^{(i)}\}_{i=1}^\Jeki$ we choose the resulting optimal design as 
\[p_n^\dagger = \argmax_i \widetilde{\mathrm{EIG}}_n(p_n^{(i)})\,.\] This design results in a new observation at time  $t_n$  from a fixed ground truth  $u^\dagger$  as:
\begin{equation}\label{eqn:1dheat_inv_prob_sequential}
y_n^\dagger = G_n(u^\dagger, p_n^\dagger, t_n) + \eta_n, \qquad \eta_n \sim \mathcal{N}(0, \Gamma),
\end{equation}
where  $G_n(u, p, t_n) := \mathcal{O}_{t_n} \circ \mathcal{G}  $ denotes the time-dependent forward operator, composed of the heat equation solver  $\mathcal{G}$  and the observation operator  $\mathcal{O}_{t_n}$  applied at time  $t_n$. \\
We perform $N=15$ time steps, resulting in \( T = 0.075 \) and observe the system at the $5$th, $10$th and $15$th timestep. Furthermore, we set $\alpha=10, \sigma=0.1$ and consider $J=\num{1e3}$ particles and solve our EKI up to time $\Teki=\num{1e4}$ with $\Jeki=3$ particles. Moreover, we choose a Gaussian prior of $\mathcal{N}(2 \cdot \mathbf{1}_d,\ 0.5 \cdot I_d)$ and a noise covariance of $\Gamma = cI_k$, where $c=0.1\|G(0.5,1)\|_2$.

\subsubsection{One dimensional parameter and observation space}

We consider two experimental setups. The first consider a one-dimensional parameter space, as well as only observation point, thus we have $d=1$ and $k=1$. We also choose for our diffusion coefficient a constant $c=1.5$, this leads to a higher non-linearity. We illustrate the upper and lower bounds of the EIG, where we compute the lower bound with a Gaussian approximation and a Laplace approximation

\begin{figure}[htbp]
    \centering
        \includegraphics[width=1\linewidth]{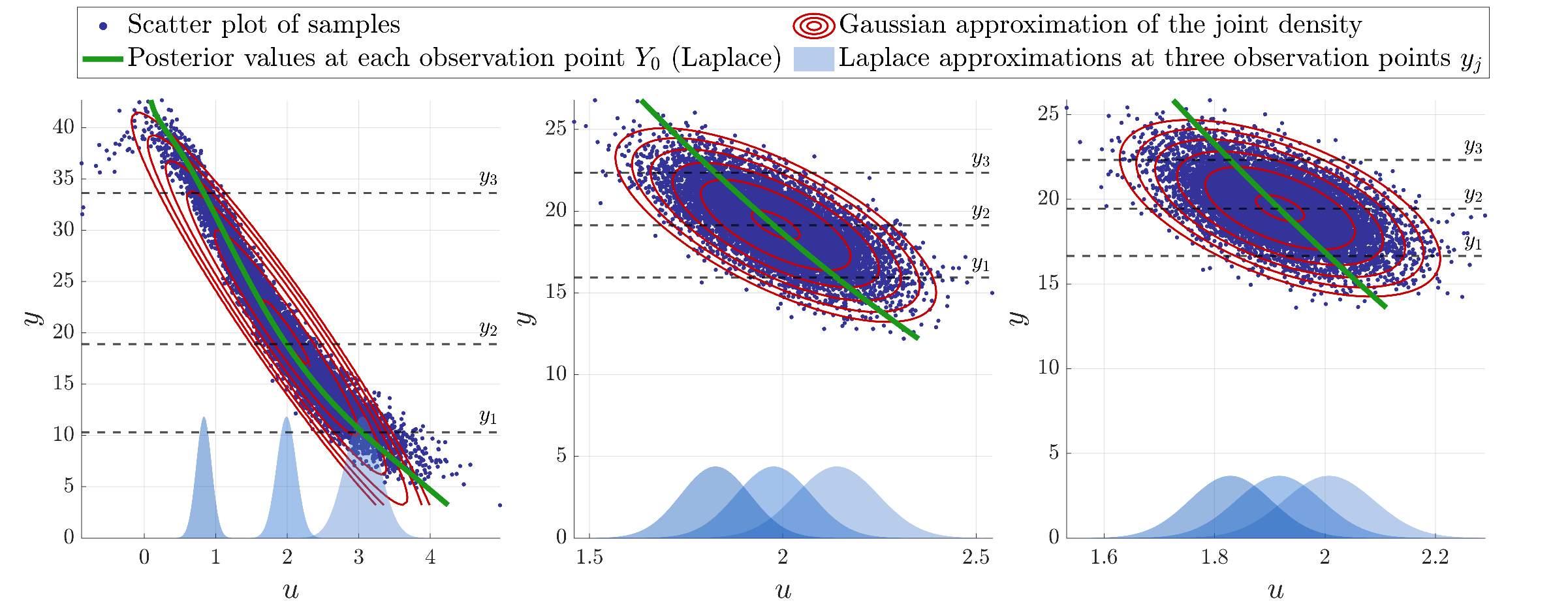}  

    \caption{Nonlinear example one dimensional parameter and observation. Gaussian approximations at steps $n=1,2,3$ as well as the mean values of the Laplace approximation. On the x-axis we have the Laplace approximation of three observation points $y_1,y_2,y_3$.}
    \label{fig:optim_nonlinear_laplace_ga_with_laplace}
\end{figure}

\begin{figure}[htbp]
    \centering
        \includegraphics[width=1\linewidth]{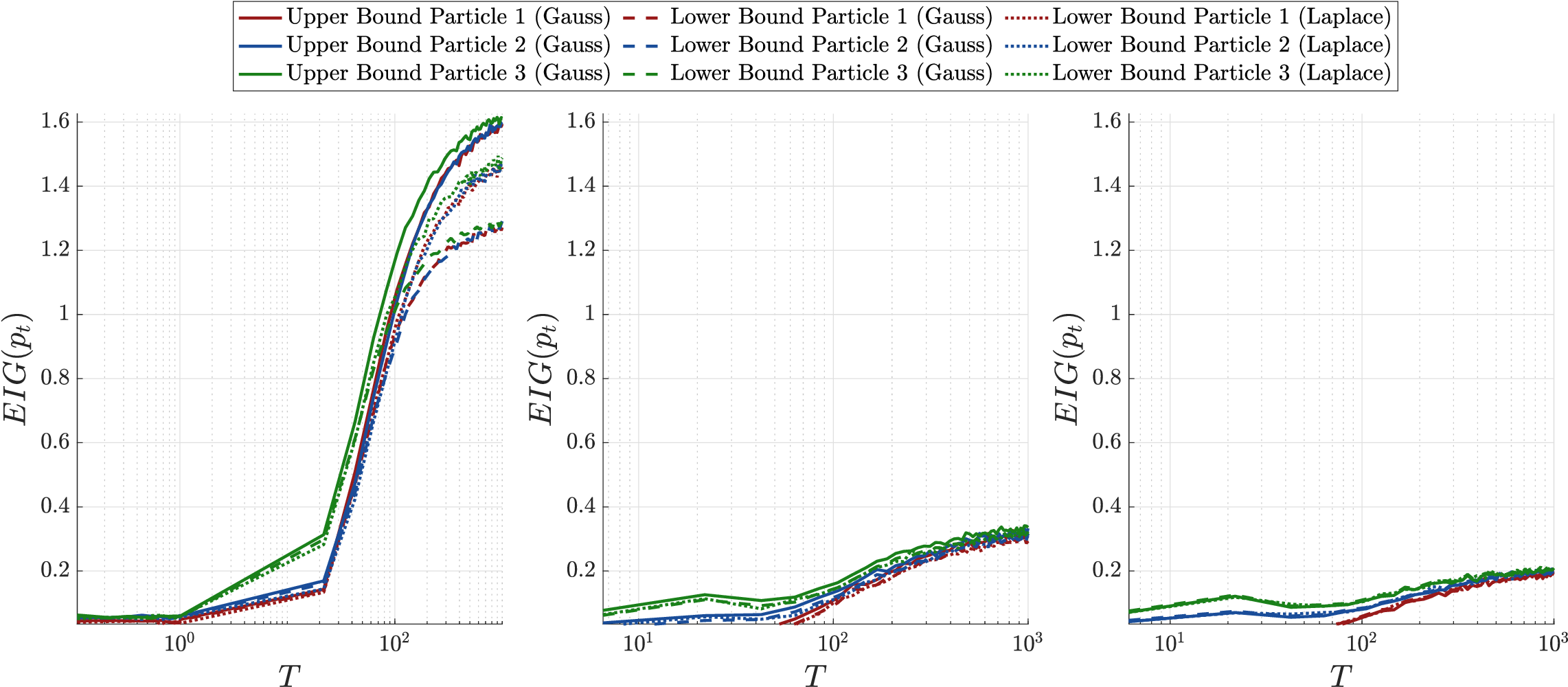}  

    \caption{Nonlinear example one dimensional parameter and observation. Upper and lower bound using Gaussian and Laplace approximations for lower bound at steps $n=1,2,3$.}
    \label{fig:optim_nonlinear_ub_lb_1d}
\end{figure}

\Cref{fig:optim_nonlinear_laplace_ga_with_laplace,fig:optim_nonlinear_ub_lb_1d} illustrate the results of the first experiment. We notice that at $n=1$ the samples do not behave perfectly Gaussian, but are slightly non-linear. Here, we can see again how due to the non-linearity the Laplace approximation performs better than the Gaussian approximation. In steps $n=2$ and $n=3$ the samples are more Gaussian and the Gaussian approximation is a good fit. Thus, the Laplace approximation and Gaussian approximation perform similarly good.\\

\subsubsection{Multidimensional parameter and observation space}

In this setting we consider a spacial discretization of $N_x=2^5+1$, a $d=8$ dimensional diffusion operator $u$ and use $k=2^{N_\text{obs}}-1$ observation points, where $N_\text{obs}=2$, hence $k=3$. Moreover, we choose in our diffusion term $c=0.5$ for a more linear model. Here we only use the Gaussian approximation for the lower bound.\\

\begin{figure}[htbp]
    \centering
        \includegraphics[width=1\linewidth]{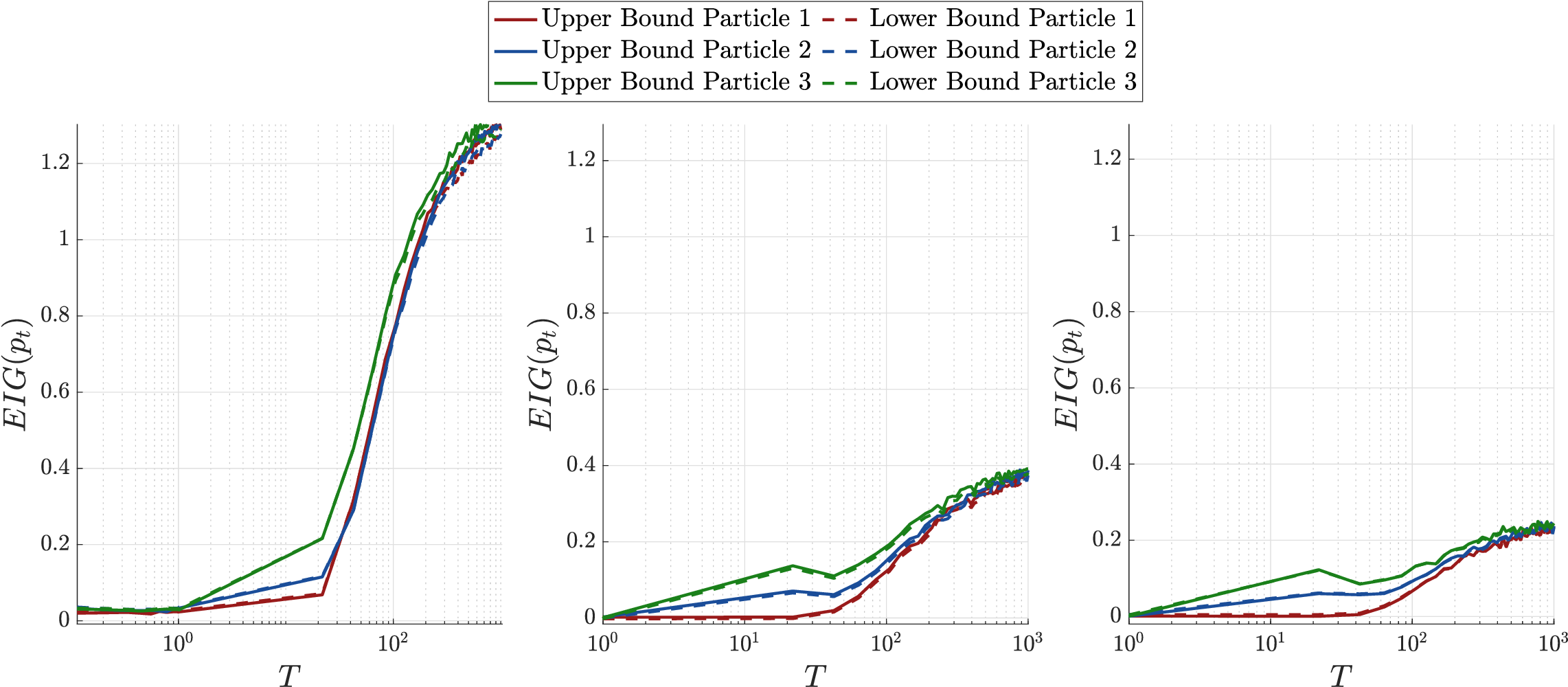}  

    \caption{Nonlinear example in multidimensional setting with $d=8$ and $k=3$: upper and lower bound of EIG for each EKI particle at steps $n=1,2,3$ using Gaussian approximation for upper and lower bound.}
    \label{fig:optim_nonlinear_ub_lb}
\end{figure}

\noindent
\Cref{fig:optim_nonlinear_ub_lb} illustrates the results for the multidimensional setting. The images display the corresponding upper and lower bounds on the EIG. We observe that the EKI updates lead to a clear increase in informativeness of the selected designs. Overall, the results confirm that the proposed gradient-free BOED framework remains effective even in nonlinear, PDE-based settings, producing consistent and increasingly informative experimental designs.

\section{Conclusions}\label{sec:conclusions}
We presented a gradient-free framework for sequential BOED that integrates EKI for optimizing design parameters and ALDI for posterior sampling. By avoiding gradient evaluations and adjoint computations, the method is particularly well-suited for PDE-constrained inverse problems and black-box forward models. To estimate the EIG, we employed Gaussian and Laplace approximations, providing practical upper and lower bounds that enable scalable and adaptive design updates.

Future work will focus on improving the quality of the EIG bounds through more flexible and adaptive approximations. While current estimates rely on Gaussian surrogates, more expressive families—such as Gaussian mixtures—may yield tighter bounds in nonlinear or multi-modal settings. Furthermore, extending the framework to account for model misspecification and non-Gaussian noise would significantly enhance its robustness and applicability in real-world experimental scenarios.

\section*{Acknowledgments}

MH and CS are grateful for the support from MATH+ project EF1-20: Uncertainty Quantification and Design of Experiment for Data-Driven Control, funded by the Deutsche Forschungsgemeinschaft (DFG, German Research
Foundation) under Germany's Excellence Strategy – The Berlin Mathematics
Research Center MATH+ (EXC-2046/1, project ID: 390685689).
RG acknowledges
support by the DFG MATH+ project AA5-5 (was EF1-25) - Wasserstein Gradient Flows for Generalised Transport in
Bayesian Inversion. Finally, we thank \cite{Bennett2020} for allowing us to use a high-performance computer to perform some of the experiments.

\appendix
\section{Proofs}
\begin{theorem}\label{thm:laplace_seq_appendix}
Let \cref{assum:LA_0} be satisfied. Then, 
\begin{equation}
D_{\text{KL}}(\nu_{u_{\varepsilon}\mid y^{({\varepsilon})},p} \mid\mid \mathcal L_{{u_{\varepsilon}\mid y^{({\varepsilon})},p}})\in\mathcal{O}({\varepsilon}^{1/2})\,.
\end{equation}

\end{theorem}
\begin{proof}
We will denote for the proof the unnormalized density of the true posterior as $\pi_{\varepsilon}$ and the normalization constant as $Z_{\varepsilon}$.\\
Recall the lower bound of EIG
\begin{equation}
\text{EIG}_u(p) 
= \mathbb{E}_{u,y^{({\varepsilon})}\mid p} \left[ \log \frac{\frac{1}{\tilde Z_{\varepsilon}}\tilde \pi_{\varepsilon}(u)} {\pi_u(u)}\right] + \mathbb{E}_{y^{({\varepsilon})} \mid p} \left[ D_{\text{KL}}(\mu_{{\varepsilon}} \mid\mid \mathcal L_{{u_{\varepsilon}\mid y^{({\varepsilon})},p}}) \right]\,. \notag
\end{equation}
We will show in the following that the Kullback-Leibler divergence between the posterior and the Laplace approximation converges to $0$ for $\varepsilon\to 0$.

\begin{align*}
0 &\le \left[ D_{\text{KL}}(\nu_{u_{\varepsilon}\mid y^{({\varepsilon})},p} \mid\mid \mathcal L_{{u_{\varepsilon}\mid y^{({\varepsilon})},p}}) \right]\\
&= \int \log \left(\frac{\frac {1}{Z_{\varepsilon}} \pi_{\varepsilon}(x)} {\frac {1}{\tilde Z_{\varepsilon}} \tilde \pi_{\varepsilon}(x)}\right) \frac{\frac {1}{Z_{\varepsilon}} \pi_{\varepsilon}(x)} {\frac {1}{\tilde Z_{\varepsilon}} \tilde \pi_{\varepsilon}(x)}  \frac {1}{\tilde Z_{\varepsilon}} \tilde \pi_{\varepsilon}(x) \mathrm dx\\
&= \int \left(\log \frac {\tilde Z_{\varepsilon}}{Z_{\varepsilon}} +\log \frac{ \pi_{\varepsilon}(x)} { \tilde \pi_{\varepsilon}(x)} \right)\frac{\tilde Z_{\varepsilon}}{Z_{\varepsilon}}\frac{ \pi_{\varepsilon}(x)} {  \tilde \pi_{\varepsilon}(x)}  \frac {1}{\tilde Z} \tilde \pi_{\varepsilon}(x) \mathrm dx\\
&= \log \frac {\tilde Z_{\varepsilon}}{Z_{\varepsilon}} \int \frac{\tilde Z_{\varepsilon}}{Z_{\varepsilon}}\frac{ \pi_{\varepsilon}(x)} {  \tilde \pi_{\varepsilon}(x)}  \frac {1}{\tilde Z} \tilde \pi_{\varepsilon}(x) \mathrm dx + \int \log \frac{ \pi_{\varepsilon}(x)} { \tilde \pi_{\varepsilon}(x)} \frac{\tilde Z_{\varepsilon}}{Z_{\varepsilon}}\frac{ \pi_{\varepsilon}(x)} {  \tilde \pi_{\varepsilon}(x)}  \frac {1}{\tilde Z} \tilde \pi_{\varepsilon}(x) \mathrm dx\\
&={\log \frac {\tilde Z_{\varepsilon}}{Z_{\varepsilon}}}+ \frac{\tilde Z_{\varepsilon}}{Z_{\varepsilon}} \int \log \frac{ \pi_{\varepsilon}(x)} { \tilde \pi_{\varepsilon}(x)} \frac{ \pi_{\varepsilon}(x)} {  \tilde \pi_{\varepsilon}(x)}  \frac {1}{\tilde Z} \tilde \pi_{\varepsilon}(x) \mathrm dx\\
&\le \log \frac {\tilde Z_{\varepsilon}}{Z_{\varepsilon}}+ \frac{\tilde Z_{\varepsilon}}{Z_{\varepsilon}} \int \frac{ \pi_{\varepsilon}(x)} { \tilde \pi_{\varepsilon}(x)} \left(\frac{ \pi_{\varepsilon}(x)} {  \tilde \pi_{\varepsilon}(x)}-1\right)  \frac {1}{\tilde Z} \tilde \pi_{\varepsilon}(x) \mathrm dx\\
&\le \log \frac {\tilde Z_{\varepsilon}}{Z_{\varepsilon}}+ \frac{\tilde Z_{\varepsilon}}{Z_{\varepsilon}} \left\|\frac{ \pi_{\varepsilon}(x)} { \tilde \pi_{\varepsilon}(x)}\right\|_{ L^2(\mathcal L_{\nu_{\varepsilon}})} \left\|\frac{ \pi_{\varepsilon}(x)} { \tilde \pi_{\varepsilon}(x)}-1 \right \|_{L^2(\mathcal L_{\nu_{\varepsilon}})}\,.
\end{align*}

Note that we used the inequality $\log(x)\leq x-1$ for $x\geq 0$.
We bound the second norm following \cite[Lemma 1]{Laplace}, we thus split the integrand into 
\begin{align*}
	J_1({\varepsilon}) & :=\int_{B_r(u_{\varepsilon}) } \left|\exp(-\frac1{\varepsilon}R_{\varepsilon}(x)) -1 \right|^{2}\ \mathcal{L}\nu_{u_{\varepsilon}\mid y^{({\varepsilon})},p}(\d x),\\
	J_2({\varepsilon}) & := \int_{B^c_r(u_{\varepsilon})} \left|\exp(-\frac1{\varepsilon} R_{\varepsilon}(x)) -1 \right|^2\ \mathcal{L}\nu_{u_{\varepsilon}\mid y^{({\varepsilon})},p}(\d x)\,,
\end{align*}
where $R_{\varepsilon}(x) := I_{\varepsilon}(x) - \widetilde I(x) = I_{\varepsilon}(x) - I_{\varepsilon}(u_{\varepsilon}) - \frac 12 \|x-u_{\varepsilon}\|^2_{C^{-1}_{\varepsilon}}$ and $r>0$ is a radius which will be specified later. 
{\scriptsize\begin{align*}
	&J_1({\varepsilon}) 
	 \leq \int_{B_r(u_{\varepsilon})} \left(\exp(\frac1{\varepsilon} c_3 \ \|x-u_{\varepsilon}\|^3) -1 \right)^2 \mathcal{L}_{\nu_{u_{\varepsilon}}\mid y^{({\varepsilon})},p}(\d x)\\
	 &  \leq \int_{B_r(u_{\varepsilon})} \left(1 - \exp(- \frac1{\varepsilon} c_3 \ \|x-u_{\varepsilon}\|^3) \right)^2\ \exp\left(-\frac1{\varepsilon} (\frac 12 \|x-u_{\varepsilon}\|^2_{C^{-1}_{\varepsilon}} - c_3 \|x-u_{\varepsilon}\|^3)\right)\ \frac{\d x}{\tilde Z_{\varepsilon}}\,.
\end{align*}}
with $c_3 > 0$ such that for sufficiently small $\varepsilon$, i.e., $\varepsilon\leq \varepsilon_r$, we have
\[
	|R_{\varepsilon}(x)| \leq c_3\ \|x-u_{\varepsilon}\|^3
	\qquad \forall x\in B_r(u_{\varepsilon})\,. 
\]
By introducing the auxiliary Gaussian measure $\rho_{{\varepsilon}} := \mathcal{N} \left(0, \varepsilon\frac 1{\gamma} I\right)$, we get
\begin{align*}
	\int_{\mathbb R^d} \left(1 - \e^{- \frac1{\varepsilon}c_3 \ \|x\|^3} \right)^2\ & \rho_{\varepsilon}(\d x)
	\leq
	\int_{\mathbb R^d} (\frac1{\varepsilon}c_3)^2 \ \|x\|^{6}\  \rho_{\varepsilon}(\d x)
	\\
    &= {\varepsilon}^{-2} (c_3)^2 \ \ev{\|(\gamma \frac1{\varepsilon})^{-1/2} \xi\|^{6}}= {\varepsilon} \frac{c^2_3}{\gamma^{6/2}} \ \ev{\|\xi\|^3} \in \mathcal O({\varepsilon})\,.
\end{align*}
This yields $J_1({\varepsilon})\in \mathcal O({\varepsilon})$ for the particular choice $r = \frac \gamma{2c_3}$. 
A close inspection of the proof in \cite{Laplace} of the second integral shows that it can be bounded completely analogously, i.e.
\begin{equation*}
	J_2({\varepsilon}) 
	\in \mathcal O(\e^{- \frac1{\varepsilon} \delta_r} {\varepsilon}^{-d/2} )\,.
\end{equation*}
This gives
\begin{equation}
\left\|\frac{ \pi_{\varepsilon}(x)} { \tilde \pi_{\varepsilon}(x)}-1 \right \|_{L^2(\mathcal L_{{u_{\varepsilon}\mid y^{({\varepsilon})},p}})}\in \mathcal O({\varepsilon}^{\frac12})\,,
\end{equation}
which implies
\begin{equation}
\left\|\frac{ \pi_{\varepsilon}(x)} { \tilde \pi_{\varepsilon}(x)} \right \|_{L^2(\mathcal L_{{u_{\varepsilon}\mid y^{({\varepsilon})},p}})}\to 1\,.
\end{equation}
Thus, for $n$ large enough and with $\frac{\tilde Z_{\varepsilon}}{Z_{\varepsilon}} \sim 1+c {\varepsilon}$ for a positive constant $c>0$ (cp. \cite[Equation (27)]{Laplace}, there holds
\begin{equation}
    D_{\text{KL}}(\nu_{u_{\varepsilon}\mid y^{({\varepsilon})},p} \mid\mid \mathcal L_{{u_{\varepsilon}\mid y^{({\varepsilon})},p}}) \in \mathcal O({\varepsilon}^{\frac12})\,. 
\end{equation}

\end{proof}
\begin{lemma}\label{lem:samples_gauss}
    Let $\{X_i\}_{i=1}^J$ be i.i.d. samples from a multivariate Gaussian \\ $\mathcal N(m,C)$ in $\R^d$. We define
\begin{align*}
    \tilde m^{(J)} &= \frac{1}{J}\sum_{i=1}^JX_i\\
    \tilde C^{(J)} &= \frac{1}{J-1}\sum_{i=1}^J (X_i - \tilde m^{(J)})\cdot  (X_i - \tilde m^{(J)})^\top
\end{align*}
Then $$D_\text{KL}(\mathcal N(m,C)\mid\mid\mathcal N(\tilde m^{(J)},\tilde C^{(J)})) \xrightarrow{L^2}0$$ and $$D_\text{KL}(\mathcal N(\tilde m^{(J)},\tilde C^{(J)})\mid\mid\mathcal N(m,C)) \xrightarrow{L^2}0$$ with rate $\mathcal O(J^{-1})$.
\end{lemma}
\begin{proof}

Without restriction we only consider $J >0$ and restrict our attention to the case where $\tilde C^{(J)}$ has full rank (which has probability $1$). We can exclude the zero set of degenerate empirical covariance matrix from all ``convergence in $L^2$'' computations below by an argument of the form
\begin{equation*}
    \E \|X_n-X\|^2 = \E \|X_n - X\|^2\chi_A + \E \|X_n - X\|^2\chi_A^\top = \E \|X_n - X\|^2\chi_A\,,
\end{equation*}
if $\mathbb P(A) = 1$. Nevertheless we use the pseudoinverse $\tilde C^{(J),\dagger}$ instead of the true inverse below.

Note that $\tilde m^{(J)}, \tilde C^{(J)}$ are unbiased estimators of $m,C$, i.e. $\E \tilde m^{(J)} = m$ and $\E \tilde C^{(J)} = C$. By the central limit theorem, $J^{-1/2}(\tilde m^{(J)} - \E \tilde m^{(J)}) = J^{-1/2}(\tilde m^{(J)} - m) \to \mathcal N(0, C)$ and $J^{-1/2}(\tilde C^{(J)} - C) \to \mathcal N(0, \cdots)$

The KL divergences are explicitly given by

\begin{align*}
    D_{KL}&\left(\mathcal N(\tilde m^{(J)},\tilde C^{(J)}) \mid\mid \mathcal N(m,C)\right)\\
    &= \frac{1}{2}\left(\tr(C^{-1} \tilde C^{(J)}) - d + \| m-\tilde m^{(J)}\|_C^2 - \log\det(C^{-1} \tilde C^{(J)}) \right)\\
    D_{KL}&\left(\mathcal N(m,C) \mid \mid\mathcal N(\tilde m^{(J)},\tilde C^{(J)}) \right)\\
    &= \frac{1}{2}\left(\tr( \tilde C^{(J),\dagger}C) - d + \| m-\tilde m^{(J)}\|_{ \tilde C^{(J)}}^2 - \log\det(\tilde C^{(J),\dagger}C) \right)\,.
\end{align*}
By Gaussianity, $\tilde m^{(J)}-m\sim \mathcal N(0, C/J)$, so $J^{1/2}C^{-1/2}(\tilde m^{(J)}-m)\sim \mathcal N(0,I)$ and $J\|\tilde m^{(J)} - m\|_C^2 \sim \chi^2(d)$. 
This means that $\E\|\tilde m^{(J)} - m\|_C^2 = \frac{d}{J}$ and $\Var \|\tilde m^{(J)} - m\|_C^2 = \frac{2d}{J^2}$. Next we will show that  $C^{-1/2}\tilde C^{(J)}C^{-1/2} \xrightarrow{L^2(\Omega,(\R^{d\times d},\|\cdot\|_F))}I_{d\times d}$. Indeed,
\begin{align*}
    C^{-1/2}\tilde C^{(J)}C^{-1/2} &= C^{-1/2}\frac{1}{J-1}\sum_{i=1}^J (X_i - \tilde m^{(J)})\cdot  (X_i - \tilde m^{(J)})^\top C^{-1/2}\\
    &=C^{-1/2}\frac{1}{J-1}\sum_{i=1}^J (X_i - m)\cdot  (X_i - m)^\top C^{-1/2}\\
    &- \frac{1}{J(J-1)}\sum_{i,j=1}^J\left(C^{-1/2} (X_i - m) \cdot (C^{-1/2}(X_j-m))^\top  \right)\\
    &= \frac{1}{J-1}\sum_{i=1}^J Y_i Y_i^\top - \frac{1}{J(J-1)}\sum_{i,j=1}^J Y_i Y_j^\top \\
    &= \frac{1}{J}\sum_{i=1}^J Y_i Y_i^\top - \frac{1}{J(J-1)}\sum_{i\neq j=1}^J Y_i Y_j^\top\,.
\end{align*}
Then 
\begin{equation*}
     C^{-1/2}\tilde C^{(J)}C^{-1/2} -I =  \frac{1}{J}\sum_{i=1}^J (Y_i Y_i^\top-I) - \frac{1}{J(J-1)}\sum_{i\neq j=1}^J Y_i Y_j^\top\,,
\end{equation*}
which has mean $0$ and second order moments are given by
\begin{align*}
    \E \left\|C^{-1/2}\tilde C^{(J)}C^{-1/2} -I\right\|_F^2 &= \E \tr\left( (C^{-1/2}\tilde C^{(J)}C^{-1/2} -I)^\top C^{-1/2}\tilde C^{(J)}C^{-1/2} -I \right)\\
    &= \E \tr \left(\frac{1}{J^2}\sum_{i,j=1}^J (Y_i Y_i^\top-I)^\top (Y_j Y_j^\top-I)\right)\\
    &+ \E\tr\left(\frac{1}{J^2(J-1)^2}\sum_{i\neq j=1}^J\sum_{k\neq l=1}^J Y_i Y_j^\top Y_k Y_l^\top \right)\\
    &-2 \E\tr\left(\frac{1}{J^2(J-1)}\sum_{k=1}^J\sum_{i\neq j=1}^J(Y_k Y_k^\top - I)^\top Y_i Y_j^\top \right)\\
    &= 0 + \frac{d}{2J(J-1)} + 0\\
    &=\mathcal O(J^{-2})\,,
\end{align*}
so $C^{-1/2}\tilde C^{(J)}C^{-1/2} \xrightarrow{L^2(\Omega,(\R^{d\times d},\|\cdot\|_F))}I_{d\times d}$. Multiplying by $C^{1/2}$ on both sides implies $\tilde C^{(J)} \xrightarrow{L^2(\Omega,(\R^{d\times d},\|\cdot\|_F))}C$, with the same order of convergence. Theorem 4.2 in \cite{rakovcevic1997continuity} shows that $\tilde C^{(J),\dagger} \to C^{-1}$ (again in $L^2$). This is due to the fact that almost surely, for $J>d$, the rank of $\tilde C^{(J)}$ is full.
We can argue $\tr(C^{-1}\tilde C^{(J)}-I) = \tr(C^{-1}(\tilde C^{(J)}-C))\leq \|C^{-1}\|_F \, \|\tilde C^{(J)}-C\|_F $ and $\tr(\tilde C^{(J),\dagger}C-I) = \tr((\tilde C^{(J),\dagger}-C^{-1})C)\leq_F  \|\tilde C^{(J),\dagger}-C^{-1}\|_F \,\|C\|_F $, so 
\begin{equation*}
    \tr(\tilde C^{(J),\dagger}C-I) \xrightarrow{L^2(\Omega,(\R^{d\times d},\|\cdot\|_F))} 0,\quad \tr(C^{-1}\tilde C^{(J)}-I) \xrightarrow{L^2(\Omega,(\R^{d\times d},\|\cdot\|_F))} 0\,,
\end{equation*}
again with the same order of convergence.

For the last remaining term in the KL divergence we note that by first order Taylor expansion,
\begin{equation}
    \log \det \tilde C^{(J)} = \log \det C + \tr\left(C^{-1}(\tilde C^{(J)}-C) \right) + \mathcal O(\|\tilde C^{(J)}-C\|^2)\,,
\end{equation}
so
\begin{equation}
    \log \frac{\det \tilde C^{(J)}}{\det C} = \tr\left(C^{-1}\tilde C^{(J)}-I \right) + \mathcal O(\|\tilde C^{(J)}-C\|^2)\,,
\end{equation}
where we already know that the mean of the first term is $0$, and has variance $2d/J$.

Collecting terms, we see that $D_{KL}\left(\mathcal N(\tilde m^{(J)},\tilde C^{(J)}) \mid\mid \mathcal N(m,C)\right)$ converges to $0$ in $L^2$, with rate $J^{-1}$. The other KL divergence is shown to converge with arguments along the same lines.

\end{proof}
\begin{corollary}\label{cor:samples_gauss}
Let $\{(Y_i,Z_i)\}_{i=1}^J$ be i.i.d. samples from a multivariate Gaussian $\mathcal N(m,C)$ in $\R^{m+n}$, where $m=(m_Y,m_Z)$ and $C = \begin{pmatrix}
    C_Y & C_{YZ}\\ C_{YZ}^\top & C_Z
\end{pmatrix}$.    
We define 
\begin{align*}
    \tilde m_Y^{(J)} &= \frac{1}{J}\sum_{i=1}^JY_i\\
    \tilde m_Z^{(J)} &= \frac{1}{J}\sum_{i=1}^JZ_i\\
    \tilde C_Y^{(J)} &= \frac{1}{J-1}\sum_{i=1}^J (Y_i - \tilde m_Y^{(J)})\cdot  (Y_i - \tilde m_Y^{(J)})^\top\\
    \tilde C_Z^{(J)} &= \frac{1}{J-1}\sum_{i=1}^J (Z_i - \tilde m_Z^{(J)})\cdot  (Z_i - \tilde m_Z^{(J)})^\top\\\\
    \tilde C_{YZ}^{(J)} &= \frac{1}{J-1}\sum_{i=1}^J (Y_i - \tilde m_Y^{(J)})\cdot  (Z_i - \tilde m_Z^{(J)})^\top\,,
\end{align*}
as well as exact and empirical conditionals
\begin{align*}
    m_{Y|Z=z} &= m_Y + C_{YZ}C_Z^{-1}(z-m_Z)\\
    C_{Y|Z=z} &= C_Y - C_{YZ}C_Z^{-1}(C_{YZ})^\top\\
    \tilde m_{Y|Z=z}^{(J)} &= \tilde m_Y^{(J)} + \tilde C_{YZ}^{(J)} \left( \tilde C_Z^{(J)} \right)^\dagger (z - \tilde m_Z^{(J)} )\\
    \tilde C_{Y|Z=z}^{(J)} &= \tilde C_Y^{(J)}-\tilde C_{YZ}^{(J)} \left( \tilde C_Z^{(J)} \right)^\dagger\left(\tilde C_{YZ}^{(J)}\right)^\top\,.
\end{align*}
Then 
\begin{align}
    D_{KL}\left(\mathcal N(\tilde m_Z^{(J)},\tilde C_Z^{(J)}) \mid\mid \mathcal N(m_Z,C_Z)\right) &\xrightarrow[\mathcal O(J^{-1})]{L^2} 0\\
    D_{KL}\left(\mathcal N(m_Z,C_Z)\mid\mid \mathcal N(\tilde m_Z^{(J)},\tilde C_Z^{(J)}) \right) &\xrightarrow[\mathcal O(J^{-1})]{L^2} 0\\
    D_{KL}\left(\mathcal N(m_{Y|Z=z}^{(J)}, \tilde C_{Y|Z=z}^{(J)}) \mid\mid \mathcal N(m_{Y|Z=z},C_{Y|Z=z})\right) &\xrightarrow[\mathcal O(J^{-1})]{L^2} 0\\
    D_{KL}\left(\mathcal N(m_{Y|Z=z},C_{Y|Z=z}) \mid\mid \mathcal N(m_{Y|Z=z}^{(J)}, \tilde C_{Y|Z=z}^{(J)})\right) &\xrightarrow[\mathcal O(J^{-1})]{L^2} 0\,,
\end{align}
i.e. empirical marginal and conditional Gaussian distributions converge against the true Gaussian marginal and conditional distributions, in the same sense and with the same rate as in \cref{lem:samples_gauss}.
\end{corollary}
\begin{proof}
     Since the Frobenius norm is additive under block structure, we see that
    \begin{align*}
        \|C_{Y,Z}-\tilde C^{(J)}_{Y,Z}\|^2_F & = \|C_{Y}-\tilde C^{(J)}_{Y}\|^2_F+2\|C_{YZ}-\tilde C^{(J)}_{YZ}\|^2_F+\|C_{Z}-\tilde C^{(J)}_{Z}\|^2_F\\
        &\geq \|C_{Y}-\tilde C^{(J)}_{Y}\|^2_F\,,
    \end{align*}
    so the statement for the marginal covariances follows from the fact that we know that $C_{Y,Z} = \begin{pmatrix}
        C_Y & C_{YZ}\\C_{YZ}^\top & C_Z
    \end{pmatrix}$ and its empirical counterpart converge in $L^2$ with respect to the Frobenius norm, i.e. $\|C_{Y,Z}-\tilde C^{(J)}_{Y,Z}\|_F\xrightarrow{L^2}0$ by \cref{lem:samples_gauss}. 

    To show convergence of the conditional covariance matrices, we use the form of the matrix inverse (using the Schur complement)
    \begin{align*}
        C_{Y,Z}^{-1} &= \begin{pmatrix}
            C_Y - C_{YZ}C_Z^{-1}C_{YZ}^\top &\cdot \\ \cdot & \cdot
        \end{pmatrix}\,,
    \end{align*}
    where we don't need do make the exact form of the other entries explicit. From \cref{lem:samples_gauss} we know that $\|C_{Y,Z}^{-1}-\tilde C^{(J),\dagger}_{Y,Z}\|_F\xrightarrow{L^2}0$, so, again using the subadditivity of the Frobenius norm with respect to block structure, we get the same order of convergence for the upper left entries, i.e.
    \[\|C_{Y|Z=z}^{-1}-\tilde C^{(J),\dagger}_{Y|Z=z}\|_F\xrightarrow{L^2}0\,.\]
    Convergence of the KL divergence in $L^2$ then follows analogously to the proof of \cref{lem:samples_gauss}.
\end{proof}

\section{Algorithms: Gradient-free ALDI and EKI}

We describe in the following algorithms how to sample from the posterior using the gradient-free ALDI as well as apply the EKI

\begin{algorithm}[h]
  \caption{Gradient-free ALDI for posterior sampling \cite{garbuno2020affine}}\label{algo:aldi}
  \begin{algorithmic}[1]
    \Require $\left\{
      \begin{array}{ll}
        m_0 \in \mathbb{R}^d, \Sigma_0 \in \mathbb{R}^{d \times d} & \text{prior mean and covariance},\\
        \{u_{n-1}^{(j)}\}_{j=1}^J, & \text{initial particle ensemble}, \\
        \{G_\ell\}_{\ell=1}^n\colon \mathbb{R}^d \times \mathcal{D} \to \mathbb{R}^k, & \text{forward operators}, \\
        \bm{y}_n^\dagger,  \bm{p}_n^\dagger, & \text{observations and designs until time n},\\
        \Gamma \in \mathbb{R}^{k \times k}, & \text{noise covariance}, \\
        \Taldi > 0, & \text{final time}, \\
        \Delta t > 0, & \text{step size}.
      \end{array}
    \right.$
    \Ensure Final particle ensemble $\{u_{n}^{(j)}\}_{j=1}^J$

    \For{$t \in [0, \Taldi]$}
      \State Set $U_t=\left[u_t^{(1)},\ldots,u_t^{(J)}\right]$.
      \State Define drift term:
      \[
        \tilde{\mathcal{A}}(U_t) := (\widetilde{C}_{u,G})_t \Gamma^{-1}\sum_{\ell=1}^n(y_{\ell}^\dagger - G_\ell(U_t, p_{\ell}^\dagger)) 
        - (\widetilde{C}_u)_t \Sigma_0^{-1}(U_t - m_0)
      \]
      \For{$j = 1$ to $J$}
        \State Update particles:
        \[
        u_{t+\Delta t}^{(j)} = u_t^{(j)} 
        - \tilde{\mathcal{A}}(U_t)\Delta t 
        + \frac{d+1}{J}(u_t^{(j)} - \bar{u}_t)\Delta t 
        + \sqrt{2\Delta t}\, (\widetilde{C}_u)_t^{1/2} \xi_t^{(j)}
        \]
        with $\xi_t^{(j)} \sim \mathcal{N}(0, I_d)$.
      \EndFor
    \EndFor
    \State Set $\{u_{n}^{(j)}\}_{j=1}^J=\{u_{\Taldi}^{(j)}\}_{j=1}^J$
  \end{algorithmic}
\end{algorithm}

\begin{algorithm}
  \caption{Regularized EKI}\label{algo:eki_reg}

  \begin{algorithmic}[1]
  \Require $\left\{\begin{array}{ll}
              \cF & \text{nonlinear operator \cref{eq:EKI_loss}} \\ 
            \{p^{(i)}_0\}_{i=1}^\Nens & \text{initial design ensemble with ensemble size $\Nens\in\mathbb{N}$} \\
            
            \alpha, C_p & \text{regularization terms} \\
              \Teki & \text{time horizon for EKI}.
              \\
            \end{array}
    \right. $
    \Ensure Final ensemble $\{p_{\Teki}^{(i)}\}_{i=1}^\Nens$.
    \For{$t \in [0,\Teki ]$ (integrate over time)}
      \For{$i = 1$ to $\Nens$}
        \State Evaluate forward map $\cF(p_t^{(i)})$ by estimating EIG with \cref{algo:eig_estimator}.
      \EndFor
      \State Compute empirical ensemble statistics: mean $\bar{p}_t$, Forward map mean $\bar{\cF}_t$ and covariances $(\widetilde{C}_p)_t$ and $(\widetilde{C}_{p,\cF})_t$
      \For{$i = 1$ to $\Jeki$}
        \State Update each particle via:
        \begin{align*}
          \frac{\mathrm{d} p^{(i)}(t)}{\mathrm{d}t} &= (1 - \rho) \left[ -(\widetilde{C}_{p,\cF})_t  \cF(p_t^{(i)}) - (\widetilde{C}_p)_t \alpha C_p^{-1} p_t^{(i)} \right] \\
          &+ \rho \left[ -(\widetilde{C}_{p,\cF})_t  \bar{\cF}_t - (\widetilde{C}_p)_t \alpha C_p^{-1} \bar{p}_t \right]
        \end{align*}
      \EndFor
    \EndFor
  \end{algorithmic}
\end{algorithm}

\bibliographystyle{siam}
\bibliography{main}
\end{document}